\documentclass[twoside]{article}
\usepackage[preprint]{arxiv}
\usepackage{xcolor}
\colorlet{blue}{black}
\usepackage{amsmath, amsfonts, amssymb, dsfont}
\usepackage{graphicx}
\graphicspath{{../}}
\usepackage{booktabs}
\usepackage{tikz}
\usepackage{subcaption}
\usetikzlibrary{arrows.meta, positioning}
\usepackage{hyperref}
\hypersetup{hidelinks}
\usepackage[most]{tcolorbox}
\usepackage{amsthm}
\usepackage{cancel}
\usepackage[normalem]{ulem}
\usepackage[round,authoryear]{natbib}

\newtheorem{assumption}{Assumption}
\newtheorem{lemma}{Lemma}
\newtheorem{theorem}{Theorem}

\newtheorem{proposition}{Proposition}

\theoremstyle{remark}

\begin{document}

\twocolumn[

\aistatstitle{Conformal Individual Treatment Effect Estimation under Networked Interference}

\aistatsauthor{
Matteo Zecchin$^{\dagger}$ \qquad
Osvaldo Simeone$^{\S}$
}

\aistatsaddress{
$^{\dagger}$Communication Systems Department, EURECOM, Sophia Antipolis, France
\qquad\\
$^{\S}$Institute for Intelligent Networked Systems (INSI), Northeastern University London, London, UK
\\
\texttt{zecchin@eurecom.fr, o.simeone@\{nulondon.ac.uk, northeastern.edu\}}
}
]

\begingroup
\renewcommand{\thefootnote}{}
\footnotetext{ The work of O. Simeone was supported by an Open Fellowship of the EPSRC (EP/W024101/1), by the EPSRC (EP/X011852/1) and by the ERC (No. 101198347). }
\addtocounter{footnote}{-1}
\endgroup

\begin{abstract}
Conformal counterfactual prediction constructs prediction sets with finite-sample coverage guarantees for counterfactual outcomes and individual treatment effects under the no-interference assumption. In this work, we relax this assumption by allowing each unit's potential outcomes to depend on other units' treatments and covariates. In this setting, propensity-score reweighting does not restore weighted exchangeability, and existing methods may fail to achieve valid coverage. To address this issue, we develop interference-adjusted weighted conformal prediction that accounts for interference by constructing an observable upper bound on the ideal and unobserved conformal $p$-value under the target intervention. The resulting prediction sets provide finite-sample marginal coverage guarantees for counterfactual outcomes and individual treatment effects in both transductive and inductive settings. We also derive a sharper construction when intervention-induced changes in nonconformity scores are bounded. Numerical experiments show that our methods preserve nominal coverage, whereas existing methods may not.
\end{abstract}

\section{Introduction}
\subsection{Motivation}

Many practical decisions require predicting how a particular unit would have responded to an action it did not receive. A physician may wish to assess what a patient's outcome would have been under an alternative treatment \citep{zhao2012estimating}, an online platform may need to evaluate how a user would have responded to an advertising campaign \citep{gordon2019comparison}, and a network operator may need to assess what would have happened had resources been allocated to a specific user \citep{bao2017prediction,hou2025if}. Such decisions require comparing the observed factual outcome under the assigned treatment with the counterfactual outcome that would have been observed under an alternative treatment. The fundamental challenge is that only the factual outcome is observed, while the remaining potential outcomes must be predicted from data \citep{rubin1974estimating,rubin2005causal}.

The reliability of the resulting decision hinges on faithfully quantifying the uncertainty associated with the counterfactual prediction. Conformal prediction has emerged as a powerful framework for this purpose, providing model-agnostic prediction sets with finite-sample marginal coverage guarantees for missing potential outcomes and individual treatment effects \citep{shafer2008tutorial}. Under the assumption that the treatment assigned to a given unit does not affect the outcomes of other units, counterfactual prediction can be formulated as a prediction problem under covariate shift, and weighted conformal prediction (WCP) \citep{tibshirani2019conformal} can be used to construct valid prediction sets via propensity-score weighting \citep{lei2021conformal}. 

In many applications, however, a unit's outcome depends not only on its own treatment but also on the treatments assigned to other units \citep{hudgens2008toward,aronow2017,forastiere2021identification}. Vaccinating one person can change the infection risk faced by others; showing an advertisement to one user can influence the behavior of their peers; and allocating resources to one user can reduce the resources available to others. In these settings, intervening on the treatment of a target unit can alter the distribution of outcomes of other units. This distributional shift cannot be expressed via unit-level propensity weights and, as a result, WCP may fail to provide valid coverage under interference.

This phenomenon is illustrated in Figure \ref{fig:motivating_example}. The example considers a simple model of local interactions in which each of $N$ units is connected to its $k$ nearest neighbors in covariate space. Treatments are assigned independently according to the units' covariates, and each unit's outcome depends on the fraction of its treated neighbors. When $k=0$, there is no interference, and WCP attains the target coverage. As the number of neighbors $k$ increases, an intervention on the target unit alters the exposures of a growing number of units. Consequently, the discrepancy between the observational and interventional distributions increases, and WCP exhibits increasingly severe undercoverage.
\begin{figure}[t]
	\centering
	\includegraphics[width=0.48\textwidth]{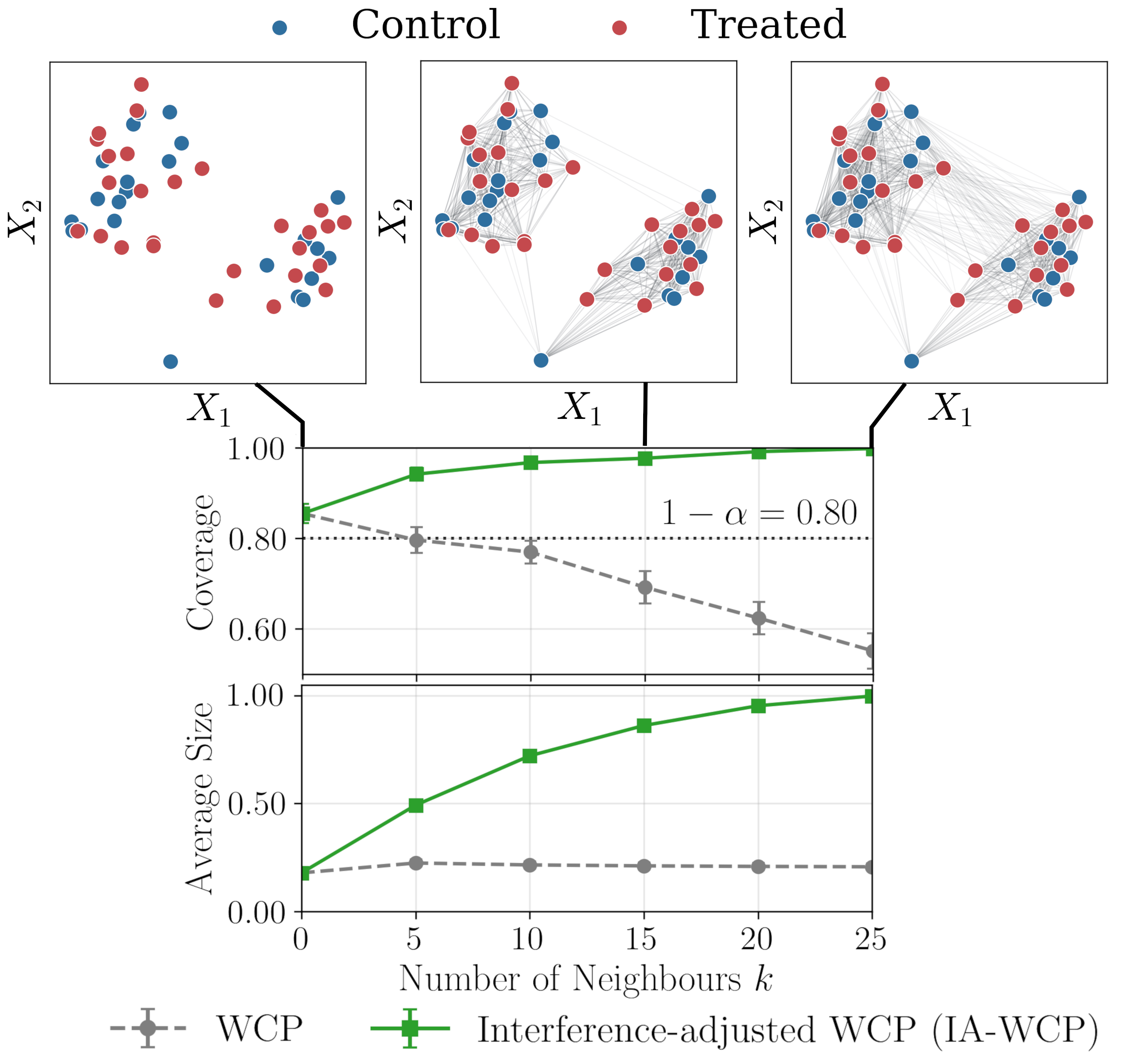}
	\caption{Effect of interference on weighted conformal prediction (WCP) \citep{lei2021conformal}. The top row shows one realization of a covariate-dependent $k$-nearest-neighbor exposure graph $\mathcal{G}(X_{1:N})$ for increasing values of $k$. The bottom row reports the empirical marginal coverage and normalized prediction-set size for a target coverage level of $1-\alpha=0.8$. WCP increasingly undercovers as interference between units increases, i.e., as $k$ increases, whereas the proposed interference-adjusted WCP (IA-WCP) maintains the target marginal coverage by returning larger prediction sets.}
	\label{fig:motivating_example}
\end{figure}

To address this problem, we develop a counterfactual conformal procedure that preserves marginal coverage under networked interference. The proposed interference-adjusted WCP (IA-WCP) is based on a corrected $p$-value computed from the observed data \citep{vovk2005algorithmic,shafer2008tutorial} and knowledge of the interference structure. As shown in Figure \ref{fig:motivating_example}, IA-WCP reduces to WCP in the absence of interference and, unlike WCP, maintains coverage as interference increases, at the cost of wider prediction sets.  We develop IA-WCP for both transductive and inductive settings \citep{clarkson2023distribution,zargarbashi2023conformal,h2024conformal} and establish finite-sample marginal coverage guarantees in each case. We also provide an example showing that the correction can be tight with respect to marginal coverage and introduce a sharper variant, interference-adjusted WCP+ (IA-WCP+), for settings in which intervention-induced changes in the affected nonconformity scores are bounded.

\subsection{Related Work}
WCP provides prediction sets for counterfactual outcomes and individual treatment effects under covariate shift and no interference \citep{tibshirani2019conformal,lei2021conformal}. Recent extensions of WCP address hidden confounding and poor overlap while retaining the no-interference assumption \citep{jin2023sensitivity,chen2024conformal,farzaneh2025synthetic,qchohi2026confounding}. Separately, research on causal inference under interference introduced the notion of exposure mappings to identify and estimate direct and spillover effects \citep{hudgens2008toward,aronow2017,forastiere2021identification}. Conformal methods for networked data have primarily targeted factual node labels or links \citep{clarkson2023distribution,huang2023uncertainty,zhao2024conformalized}. The closely related method of \cite{zhou2025adaptive} constructs conformal prediction sets for ITEs on networked data but assumes no interference. Our work connects these lines of research by constructing prediction sets with finite-sample marginal coverage guarantees for individual potential outcomes and treatment effects under network interference. A more detailed discussion is provided in Appendix \ref{app:related_work}.

\subsection{Contributions}

Our main contributions are as follows:

\begin{itemize}
    \item We formulate conformal inference for ITEs and treatment effects under interference in both the transductive setting, where the targets are units in the observed population, and the inductive setting, where the goal is to generalize to new units.

    \item We develop IA-WCP, an interference-adjusted counterfactual conformal procedure that produces prediction sets for the potential outcomes with finite-sample marginal coverage guarantees under interference, and reduces to WCP \citep{lei2021conformal} in the absence of interference.

    \item Under a score-stability condition, we present a sharper version of IA-WCP, referred to as IA-WCP+, that can produce smaller prediction sets while maintaining coverage.

    \item We provide experimental results showing that WCP can undercover under interference, whereas the proposed IA-WCP and IA-WCP+ preserve marginal coverage.
\end{itemize}

\section{Setting}

\begin{figure*}[t]
	\centering
	\includegraphics[width=0.8\textwidth]{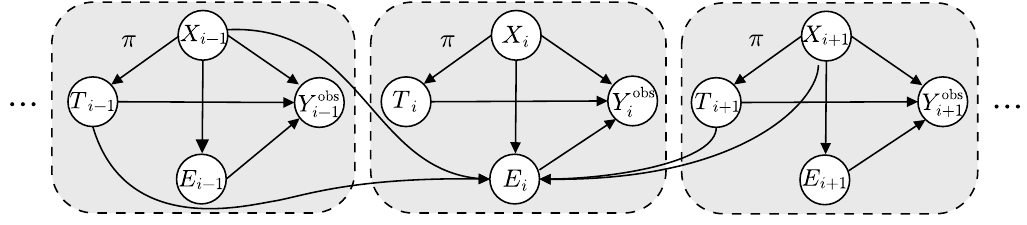}
	\caption{Graphical model of the setting considered in this work. For each unit $i$, the covariate $X_i$ is independently sampled from a common population distribution $P_X$, and the treatment $T_i$ is independently assigned according to a common covariate-dependent policy $\pi(X_i)$. The exposure $E_i$ depends on the unit's covariate $X_i$ and on the covariates $X_{j}$ and treatments $T_{j}$ of the other units $j\in\mathcal{N}_i(X_{1:N})$. The observed outcome $Y_i^{\mathrm{obs}}$ equals the potential outcome corresponding to the assigned treatment, and depends on the unit's covariate and exposure.}
	\label{fig:graphical_model}
\end{figure*}

We consider a population of $N$ units indexed by $i\in\{1,\ldots,N\}$. Each unit has covariates $X_i\in\mathcal X$ and receives a binary treatment $T_i\in\{0,1\}$. We write $X_{1:N}=(X_1,\ldots,X_N)$ and $T_{1:N}=(T_1,\ldots,T_N)$, and use $X^{-i}_{1:N}$ and $T^{-i}_{1:N}$ for the corresponding vectors with unit $i$ removed. The covariates are independently sampled from a common population distribution $P_X$, and, as detailed next, treatments are conditionally independent under a common covariate-dependent policy.
\begin{assumption}[Treatment assignment and overlap]
	\label{ass:prop_unit}
	There exists a propensity score function $\pi:\mathcal X\to(0,1)$ such that, conditionally on covariates $X_{1:N}$, the assigned treatments are independent and distributed as
	\begin{align}
		T_i\mid X_{1:N}\sim \mathrm{Bernoulli}(\pi(X_i)), \quad i=1,\dots,N.
	\end{align}
	Moreover, there exists a strictly positive constant $\underline \pi\in(0,1/2]$ such that
	\begin{align}
		\underline \pi \le \pi(x)\le 1-\underline \pi, \qquad \forall x\in\mathcal X.
	\end{align}
\end{assumption}
For a given treatment vector $T_{1:N}$, define the control and treated index sets, respectively, as
\begin{align}
	\mathcal I_0=\{j\in\{1,\dots,N\}:T_j=0\},
\end{align}
and
\begin{align}
    \mathcal I_1=\{j\in\{1,\dots,N\}:T_j=1\}.
\end{align}

For each unit $i$, let $(Y_i(0),Y_i(1))$ denote its potential outcomes under treatment $T_i=0$ and $T_i=1$. Departing from the standard SUTVA, we allow these outcomes to depend on covariates and treatments of a subset of units, excluding unit $i$ \citep{tchetgen2012causal,hudgens2008toward}. Specifically, for each unit $i$, we summarize this inter-unit dependence through an exposure variable $E_i\in\mathcal E$ \citep{aronow2017} that depends on units in a neighborhood set
\begin{align}
    \mathcal N_i(X_{1:N})  \subseteq  \{1,\ldots,N\}\setminus\{i\}.
\end{align}

The neighborhood $\mathcal N_i(X_{1:N})$ generally depends on the covariates $X_{1:N}$. For example, in Figure \ref{fig:motivating_example} the set $\mathcal N_i(X_{1:N})$ corresponds to the $k$ nearest neighbors to unit $i$. Furthermore, the neighborhood relation may be directed, meaning that $j\in\mathcal N_i(X_{1:N})$ need not imply $i\in\mathcal N_j(X_{1:N})$. Overall, the neighborhood relations $\{\mathcal N_i(X_{1:N})\}^N_{i=1}$ define a covariate-dependent exposure (directed) graph $\mathcal G(X_{1:N})$ with nodes given by the units $\{1,\dots,N\}$ and edges $(i,j)$ with $j\in \mathcal N_i(X_{1:N})$.

In this work, we restrict attention to permutation-equivariant neighborhood rules. For a permutation $\sigma$ of the set $\{1,\dots,N\}$, write
\begin{align}
    T_{\sigma(1:N)}=(T_{\sigma(1)},\dots,T_{\sigma(N)}), 
\end{align}
and
\begin{align}
 \quad X_{\sigma(1:N)}=(X_{\sigma(1)},\dots,X_{\sigma(N)}).
\end{align}

\begin{assumption}[Permutation-equivariant neighborhood rule]
    \label{ass:eq_exp_map}
    For any permutation $\sigma$ of the set $\{1,\dots,N\}$, and every unit $i\in\{1,\dots,N\}$, the neighborhood rule satisfies
    \begin{align}
    \mathcal N_i(X_{\sigma(1:N)})=\left\{j:\sigma(j)\in\mathcal N_{\sigma(i)}(X_{1:N})\right\}.
        \label{eq:neighborhood_equivariance}
    \end{align}
\end{assumption}
\paragraph{Example:}
Given a distance $d:\mathcal X\times\mathcal X\to\mathbb R_+$ and a radius $r\geq 0$, the radius-$r$ neighborhood rule corresponds to $\mathcal N_i(X_{1:N})  =   \left\{  j\neq i:d(X_i,X_j)\leq r \right\}$.

This rule satisfies Assumption \ref{ass:eq_exp_map} since permutations of units preserve pairwise distances. For illustration, consider $X_1=0$, $X_2=0.5$, and $X_3=1$, with $d(x,x')=|x-x'|$, $r=0.6$, and a permutation $ \sigma(1)=2, \sigma(2)=3, \sigma(3)=1$.

Then, we have $X_{\sigma(1:3)}=(X_2,X_3,X_1)$ and, for $i=1$, the neighborhood set is
$\mathcal N_1(X_{\sigma(1:3)}) =\left\{j\neq 1:   d(X_{\sigma(1)},X_{\sigma(j)})\leq 0.6\right\}=\{2,3\}$, which can be seen to satisfy Assumption \ref{ass:eq_exp_map}. More examples of neighborhood rules satisfying this property are given in Appendix \ref{sec:perm_eq_maps}.
 
We define the multiset of treatments and covariates for units in neighborhood $\mathcal N_i(X_{1:N})$ as
\begin{align}
    \mathcal{L}_i(X_{1:N},T_{1:N})=  \bigl\{\bigl\{(X_j,T_j):j\in\mathcal N_i(X_{1:N})\bigr\}\bigr\},
    \label{eq:local_configuration}
\end{align}
where the doubled braces indicate that repeated covariate--treatment pairs are retained with their multiplicities. Each exposure variable $E_i$ depends only on $X_i$ and variables in the multiset $\mathcal{L}_i(X_{1:N},T_{1:N})$ as detailed next.

\begin{assumption}[Exposure generation and observation]
    \label{ass:exposure_mechanism}
    Conditionally on $(X_{1:N},T_{1:N})$, the exposure variables $E_1,\dots,E_N$ are independent. Moreover, there exists a common conditional distribution $P_{E\mid X, \mathcal{L}}$ such that, for each $i=1,\dots,N$, we have
    \begin{align}
        \label{eq:cond_ind_exposure}
        E_i \mid (X_{1:N},T_{1:N}) \sim  P_{E\mid X,\mathcal{L}} \left(\cdot\mid X_i,\mathcal{L}_i(X_{1:N},T_{1:N})\right).
    \end{align}
	Each exposure $E_i$ is either directly observed or can be exactly recovered from $X_{i}$ and the multiset $\mathcal{L}_i(X_{1:N},T_{1:N})$.
\end{assumption}

The assumption that the exposure variable $E_i$ is observable is natural when the exposure mechanism is known. Examples include the number or fraction of treated neighbors within a known group or cluster and distance-weighted exposures that depend on observed covariates \citep{aronow2017,forastiere2022estimating,cai2024independent}.

\begin{assumption}[No hidden interference confounding]
	\label{ass:no_hidden_interference_confounding}
	Conditional on the covariates, exposures, and treatments $(X_{1:N},E_{1:N},T_{1:N})$, where $E_{1:N}=(E_1,\dots,E_N)$, the potential outcome pairs $(Y_i(0),Y_i(1))_{i=1}^N$ are independent across units. Moreover, for $t\in\{0,1\}$, there exists a common conditional distribution $P_{Y(t)\mid X,E}$ such that, for each $i\in\{1,\dots,N\}$,
	\begin{align}
		\label{eq:potential_distribution}
		Y_i(t) \mid (X_{1:N},E_{1:N},T_{1:N}) \sim P_{Y(t)\mid X,E}(\cdot\mid X_i,E_i).
	\end{align}
\end{assumption}

\begin{assumption}[Consistency]
	\label{ass:consistency}
	For each observed unit $i$, the observation $Y_i^{\mathrm{obs}}$ equals the potential outcome evaluated at the realized treatment, i.e.,
	\begin{align}
		Y_i^{\mathrm{obs}}
		=
		\begin{cases}
			Y_i(0), & \text{if } T_i=0,\\
			Y_i(1), & \text{if } T_i=1.
		\end{cases}
	\end{align}
\end{assumption}
Overall, as a result of the observations above, each observation is distributed as
\begin{align}
	Y_i^{\mathrm{obs}} \mid (X_{1:N},T_{1:N},E_{1:N}) \sim P_{Y(T_i)\mid X,E}(\cdot\mid X_i,E_i),
\end{align}
and the joint distribution of all relevant variables is specified by the following conditional distributions
\begin{subequations}\label{eq:data_model}
\begin{align}
&X_i \sim P_X, \\
&T_i \mid X_i \sim \mathrm{Bernoulli}(\pi(X_i)), \\
&E_i \mid (X_{1:N},T_{1:N})
\sim P_{E\mid X,\mathcal{L}}\!\left(
\cdot \mid X_i,\mathcal{L}_i(X_{1:N},T_{1:N})
\right), \label{eq:exposure_distribution}\\
&Y_i(T_i)\mid (X_{1:N},E_{1:N},T_{1:N}) \sim P_{Y(T_i)\mid X,E}(\cdot\mid X_i,E_i)\\
&Y_i^{\mathrm{obs}}=Y_i(T_i)
\end{align}
\end{subequations}
Under this joint distribution, variables associated with different units are generally dependent.

\subsection{Problem Formulation}

In this section, we formalize the individual treatment effect under interference and state the corresponding counterfactual coverage objectives in the transductive and inductive settings.   The primary inferential target is the individual treatment effect (ITE), defined as the difference between a unit's two potential outcomes at the same exposure,
\begin{align}
	\label{eq:ite}
	\tau = Y(1)-Y(0).
\end{align}
Given the observational dataset $\mathcal D$, to be specified differently for the inductive and transductive settings, and a unit with covariates $X$ and exposure $E$, the goal is to construct an ITE prediction set $\Gamma^{\mathrm{ITE}}(X,E)$ satisfying the marginal coverage guarantee
\begin{align}
	\Pr[\tau\in \Gamma^{\mathrm{ITE}}(X,E)]\ge 1-\alpha,
\end{align}
for a user-defined miscoverage level $\alpha\in(0,1)$, where the probability is taken over both the observational data $\mathcal{D}$ and the test unit. We consider two instantiations of this requirement.
\subsubsection{Transductive ITE Estimation}

As illustrated in the left panel of Figure \ref{fig:tran_vs_ind}, in the transductive setting we are given a dataset $\mathcal D=((X_i,T_i,E_i,Y_i^{\mathrm{obs}}))_{i=1}^N$ for all $N$ units, and the goal is to construct a family of prediction sets $\{\Gamma_i^{\mathrm{ITE}}\}_{i=1}^N$ that cover the ITEs of the observed units on average. Specifically, the transductive ITE coverage requirement is
\begin{align}
	\label{eq:transductive_coverage_ite}
	\Pr_{I\sim\mathrm{Unif}(\{1,\dots,N\})}\hspace{-0.25em}\left[Y_{I}(1)-Y_{I}(0)\hspace{-0.2em}\in \hspace{-0.2em}\Gamma_{I}^{\mathrm{ITE}}(X_I,E_I)\right]\hspace{-0.25em}\ge\hspace{-0.25em} 1-\alpha,
\end{align}
where the probability is taken over the random index ${I}\sim\mathrm{Unif}(\{1,\dots,N\})$, as well as over the randomness in the observational dataset $\mathcal D$ and the potential outcomes.

Since the factual outcome $Y_i(T_i)$ is observed for every unit, constructing an ITE prediction set reduces to constructing a prediction set for the missing potential outcome $Y_i(\bar T_i)$, where $\bar T_i=1-T_i$ denotes the counterfactual treatment. Let $\Gamma_i^{\bar T_i}$ denote a prediction set for this missing potential outcome. Combining this set with the observed factual outcome yields the sets
\begin{align}
    \label{eq:transductive_ite_set}
	\Gamma_i^{\mathrm{ITE}}
	=
	\begin{cases}
		\Gamma_i^1 - Y_i(0), & \text{if } T_i=0,\\
		Y_i(1) - \Gamma_i^0, & \text{if } T_i=1,
	\end{cases}
\end{align}
where for sets $A\subset \mathbb{R}^d$ and $B\subset \mathbb{R}^d$, we have defined the difference set $A-B=\{a-b:a\in A, b\in B\}$.
Consequently, if the counterfactual potential-outcome prediction sets satisfy
\begin{align}
	\label{eq:transductive_coverage_counterfactual}
	\Pr_{{I}\sim\mathrm{Unif}(\{1,\dots,N\})}\left[Y_{I}(\bar T_{I})\in \Gamma_{I}^{\bar T_{I}}\right]\ge 1-\alpha,
\end{align}
then the induced sets $\{\Gamma_i^{\mathrm{ITE}}\}_{i=1}^N$ satisfy the transductive ITE coverage requirement in \eqref{eq:transductive_coverage_ite}.

\begin{figure}
	\centering
	\includegraphics[width=0.475\textwidth]{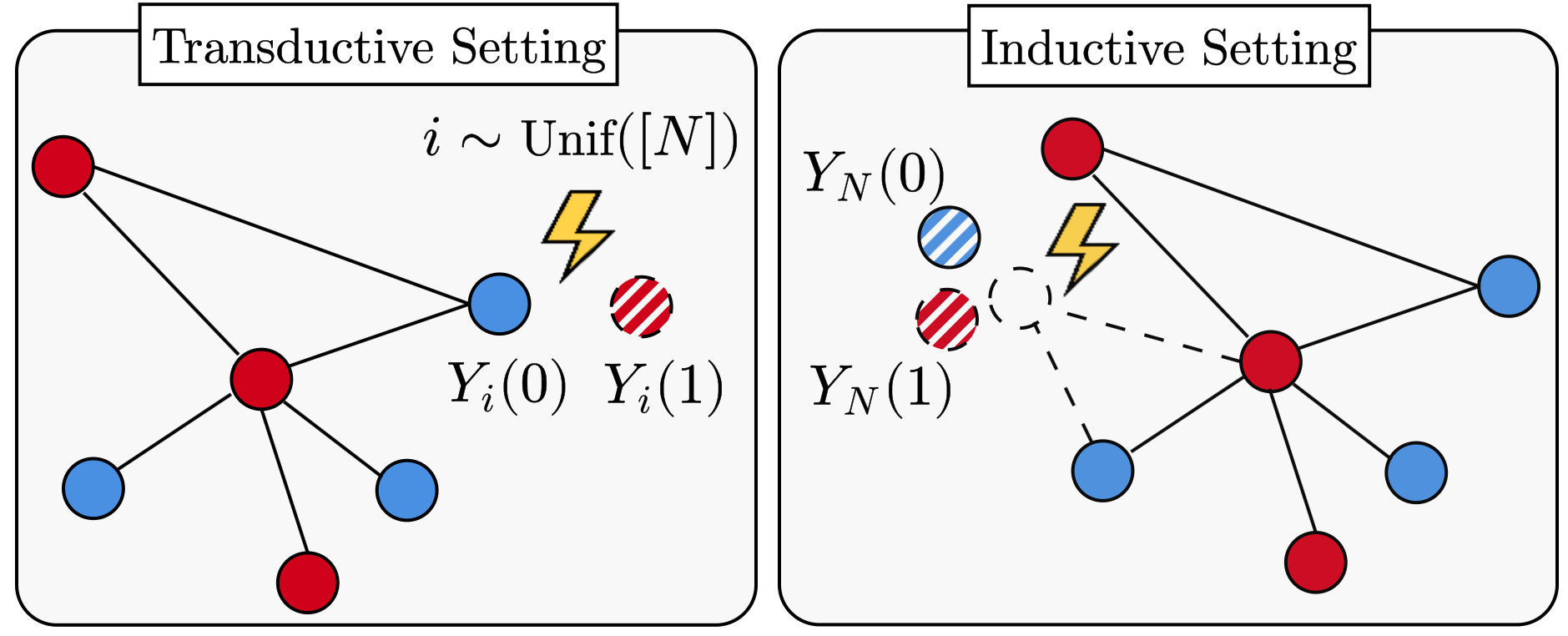}
	\caption{Illustration of the ITE estimation task in the transductive and inductive settings. In the transductive setting, in the left panel, given a network of $N$ interfering units, the goal is to estimate the potential outcome of a random unit $i$ picked uniformly at random from the observed units. In the inductive setting, in the right panel, given a network of $N-1$ interfering units, the inferential targets are the two unobserved potential outcomes associated with an additional unit sampled from the marginal distribution $P_X$ and embedded in the network.}
	\label{fig:tran_vs_ind}
\end{figure}

\subsubsection{Inductive ITE Estimation}

\label{sec:inductive_ite_estimation}
As illustrated in the right panel of Figure \ref{fig:tran_vs_ind}, in the inductive setting we are given a dataset $\mathcal{D}$ including tuples $\mathcal{D}=\{(X_i,T_i,E_i,Y_i^{\mathrm{obs}})\}_{i=1}^{N-1}$ for $N-1$ units, and the goal is to construct a set predictor $\Gamma^{\mathrm{ITE}}$ that covers the ITE of a new unit with covariates $X_{N}\sim P_X$ and an exposure level $E_{N}$ given by the exposure that this unit would have experienced in the network of $N$ interfering units whose first $N-1$ units have covariates and treatments as in $\mathcal{D}$. 

Formally, the output of inductive ITE estimation is a set predictor $\Gamma^{\mathrm{ITE}}$ satisfying the marginal coverage guarantee
\begin{align}
	\label{eq:coverage_inductive_id}
	\Pr\left[Y_{N}(1)-Y_{N}(0)\in \Gamma^{\mathrm{ITE}}(X_{N},E_{N})\right]\ge 1-\alpha,
\end{align}
where the probability is over the dataset $\mathcal{D}$ and the test unit $(X_{N},E_{N},Y_{N}(0),Y_{N}(1))$. Unlike in the transductive setting, neither potential outcome is observed for the test unit $N$. Therefore, the ITE prediction set $\Gamma^{\mathrm{ITE}}(X_{N},E_{N})$ must be constructed from prediction sets $\Gamma^0(X_{N},E_{N})$ and $\Gamma^1(X_{N},E_{N})$ for both potential outcomes.

In the following, we detail the proposed solution for the transductive setting, and present its  counterpart for inductive inference, which follows the same principles, in Appendix \ref{app:inductive_ia_wcp}.

\section{Transductive Interference-Adjusted Weighted Conformal Prediction}
\label{sec:conf_tran_ite}
In this section, we introduce IA-WCP for the transductive setting. Without loss of generality, we condition on the event that the randomly selected test unit $I\sim\mathrm{Unif}(\{1,\dots,N\})$ is a control unit, i.e., $T_I=0$, and construct a prediction set for its missing treated potential outcome $Y_I(1)$. In what follows, we suppress the dependence on the target treatment in the notation whenever it is clear from context.

\subsection{Weighted Conformal Prediction}
\label{sec:imp_int_trans}
We first review conventional WCP for counterfactual outcomes \citep{tibshirani2019conformal,lei2021conformal} and explain why it is generally invalid under interference. Following the construction in \citep{lei2021conformal}, let $S:\mathcal X\times\mathcal E\times\mathcal Y\to\mathbb R$ be a fixed nonconformity scoring function. For every treated calibration unit $j\in\mathcal I_1$, define the observational calibration score $V_j^{\mathrm{obs}} = S(X_j,E_j,Y_j(1))$,
and, for a candidate value $y\in\mathcal Y$ and test unit $I$, define the test score $V_I(y) = S(X_I,E_I,y)$.

To correct for the covariate shift induced by treatment assignment, WCP \citep{tibshirani2019conformal,lei2021conformal} assigns to each unit $j\in\mathcal I_1\cup\{I\}$ the propensity-ratio-based weight 
\begin{align}
\label{eq:prop_ratio_tran} 
W_j= \frac{1-\pi(X_j)}{\pi(X_j)},
\end{align}
and it defines the observational conformal $p$-value as
\begin{align}
\label{eq:conf_p_value_trans} 
p_{I,y}^{\mathrm{obs}} = \frac{W_I + \sum_{j\in\mathcal I_1} W_j \mathds 1\{V_j^{\mathrm{obs}}\ge V_I(y)\}}{W_I + \sum_{j\in\mathcal I_1}W_j}.
\end{align}

Finally, the prediction set is constructed as
\begin{align}
\label{eq:conf_set_trans_no_interference}
 \Gamma_I^{\rm WCP} = \{y\in\mathcal Y:p_{I,y}^{\mathrm{obs}}>\alpha\}.
\end{align}
As proved in \citep{lei2021conformal}, the marginal coverage guarantee in \eqref{eq:transductive_coverage_counterfactual} holds for the missing potential outcome $Y_I(1)$ under the standard SUTVA. In fact, in this case, the discrepancy between the distributions of calibration and test scores is fully accounted for by the covariates, and the weights in \eqref{eq:prop_ratio_tran} compensate for this shift. Under network interference, however, a residual discrepancy may remain through the exposure variables, even after correcting for covariate shift.

\subsection{Transductive IA-WCP}
\label{subsec:trans_ia_wcp}
Define the set of treated calibration units whose neighborhoods contain the test unit $I$ as
\begin{align}
\mathcal H_I = \left\{j\in\mathcal I_1:I\in\mathcal N_j(X_{1:N})\right\}.
\label{eq:structural_affected_trans}
\end{align}
IA-WCP modifies the WCP prediction set \eqref{eq:conf_set_trans_no_interference} as
\begin{align}
\label{eq:conf_set_trans_adj} 
\Gamma_I^{\mathrm{IA-WCP}} = \left\{y\in\mathcal Y: p_{I,y}^{\mathrm{obs}} > (\alpha-\rho_{I,y})_+\right\},
\end{align}
where  $(x)_+=\max\{x,0\}$ and 
\begin{align}
		\label{eq:rho_trans}
		\rho_{I,y} = \frac{\sum_{j\in\mathcal H_I} W_j \mathds 1\{V_j^{\mathrm{obs}}<V_I(y)\}}{W_I + \sum_{j\in\mathcal I_1}W_j}.
\end{align}
The correction $\rho_{I,y}$ is the fraction of total conformal weight assigned to the affected calibration units whose observational score lies below the candidate test score. Intuitively, this term captures the discrepancy between calibration and test scores caused by the presence of network interference. In particular, note that, in the absence of interference, we have $\mathcal H_I=\emptyset$ and $\rho_{I,y}=0$ for every $y\in\mathcal Y$, and the prediction set $\Gamma_I^{\mathrm{IA-WCP}}$ reduces to the WCP set \citep{lei2021conformal} in \eqref{eq:conf_set_trans_no_interference}.

The following result shows that the correction term \eqref{eq:rho_trans} is sufficient to restore the marginal guarantee \eqref{eq:transductive_coverage_counterfactual}.

\begin{theorem}
    \label{thm:cov_trans_iawcp}
    Under Assumptions \ref{ass:prop_unit}--\ref{ass:consistency}, for a randomly chosen control unit $I$, the IA-WCP set $\Gamma_I^{\mathrm{IA-WCP}}$ in \eqref{eq:conf_set_trans_adj} satisfies the inequality
    \begin{align}
    \label{eq:cov_trans_iawcp}
    \Pr\left[Y_I(1)\notin\Gamma_I^{\mathrm{IA-WCP}}\Big| T_I=0\right] \le \alpha.
    \end{align}
\end{theorem}
\begin{proof}
See Appendix \ref{app:transductive_validity}.
\end{proof}
When $T_I=1$, the analogous construction uses the control units as the calibration sample and the reverse propensity-ratio weights $W_j=\pi(X_j)/(1-\pi(X_j))$, ensuring that the resulting set \eqref{eq:transductive_ite_set} satisfies the desired condition \eqref{eq:transductive_coverage_ite}.

To explain the rationale behind the correction term \eqref{eq:rho_trans} introduced by IA-WCP, we outline the main steps in
the proof of Theorem \ref{thm:cov_trans_iawcp}. First, we condition on the event
\begin{align}
\label{eq:factual_conditioning_event}
A_{\mathcal I_1,I}
=
\big\{T_j=1 \text{ for } &j\in\mathcal I_1,T_I=0,\nonumber\\
&T_j=0 \text{ for } j\notin\mathcal I_1\cup\{I\}\big\},
\end{align}
which states that unit $I$ is a control unit and that $\mathcal I_1$ is the set of treated units. As $\mathcal I_1$ ranges over all subsets of $\{1,\ldots,N\}\setminus\{I\}$, the events $A_{\mathcal I_1,I}$ form a disjoint partition of $\{T_I=0\}$. It therefore suffices to establish \eqref{eq:cov_trans_iawcp} conditionally on each event $A_{\mathcal I_1,I}$ and then apply the law of total probability.

Conditional on $A_{\mathcal I_1,I}$, as shown in Appendix \ref{app:transductive_weighted_exchangeability}, the covariates $(X_j)_{j\in\mathcal I_1\cup\{I\}}$ are weighted exchangeable with the weights in \eqref{eq:prop_ratio_tran}. This is the key fact underlying the validity of WCP in the absence of interference.

In the presence of interference, however, this does not imply the weighted exchangeability of the calibration and test scores due to the dependence of the outcomes on the exposure variables. In particular, conditional on
$A_{\mathcal I_1,I}$ and $X_{1:N}$, the joint distribution $p\left( (Y_j(1),E_j)_{j\in\mathcal I_1\cup\{I\}} \mid A_{\mathcal I_1,I},X_{1:N}\right)$ need not be invariant under permutations that exchange unit $I$ with a calibration unit. 

IA-WCP rests on the observation that this source of non-exchangeability disappears under a modified interventional distribution in which the treatment of unit $I$ is changed to one. In fact, denoting the resulting exposures and potential outcomes by $(\tilde E_j^I)_{j\in\mathcal I_1\cup\{I\}}$ and
$(\tilde Y_j^I(1))_{j\in\mathcal I_1\cup\{I\}}$, respectively, their conditional joint distribution $p\left(
    (\tilde Y_j^I(1),\tilde E_j^I)_
    {j\in\mathcal I_1\cup\{I\}}
    \mid A_{\mathcal I_1,I},X_{1:N},
    \operatorname{do}(T_I=1)
\right)$ can be shown to be exchangeable under permutations of test and calibration units (see Appendix \ref{app:transductive_weighted_exchangeability}).  Combined with the weighted exchangeability of the covariates, this implies that the prediction set 
\begin{align}
\label{eq:ideal_WCP_set}
 \tilde\Gamma_I^{\rm WCP} = \{y\in\mathcal Y:p_{I,y}^{\mathrm{int}}>\alpha\},
\end{align}
built using the interventional conformal $p$-value 
\begin{align}
\label{eq:conf_p_value_trans_int} 
p_{I,y}^{\mathrm{int}} = \frac{{W}_I + \sum_{j\in\mathcal I_1} {W}_j \mathds 1\{V_j^{\mathrm{int}}\ge V_I^{\mathrm{int}}(y)\}}{{W}_I + \sum_{j\in\mathcal I_1}{W}_j},
\end{align}
with scores $V_j^{\mathrm{int}} = S(X_j,\tilde E_j^{I},\tilde Y_j^{I}(1))$ for $j\in\mathcal I_1$ and 
$V_I^{\mathrm{int}}(y) = S(X_I,\tilde{E}^I_I,y)$ satisfies the marginal coverage guarantee.
\begin{lemma}
     \label{thm:ideal_trans_cov}
    Under Assumptions \ref{ass:prop_unit}--\ref{ass:consistency}, for a randomly chosen control unit $I$, the ideal WCP set $\tilde\Gamma_I^{\rm WCP} $ in \eqref{eq:ideal_WCP_set} satisfies the inequality
    \begin{align}
    \label{eq:ideal_trans_cov_eq}
    \Pr\left[Y_I(1)\notin\tilde\Gamma_I^{\rm WCP} \Big| T_I=0\right] \le \alpha.
    \end{align}
\end{lemma}
\begin{proof}
	See Appendix \ref{app:transductive_weighted_exchangeability}.
\end{proof}

The last step of the proof of Theorem \ref{thm:cov_trans_iawcp} is to show that it is possible to construct a coupling
of the observational and interventional vectors of nonconformity scores, that is, a joint distribution whose marginals coincide with their respective observational and interventional distributions, such that for every $y\in\mathcal{Y}$, the following bound holds
\begin{align}
    p_{I,y}^{\mathrm{int}}\leq p_{I,y}^{\mathrm{obs}}+\rho_{I,y}, \quad \text{a.s.}
\end{align}
The correction term $\rho_{I,y}$ in \eqref{eq:rho_trans} then accounts for the discrepancy between the observable and interventional conformal $p$-values, ensuring that the set $\Gamma_I^{\mathrm{IA-WCP}}$ includes $\tilde\Gamma_I^{\rm WCP}$ almost surely, and completing the proof of Theorem \ref{thm:cov_trans_iawcp}.

\subsection{Transductive IA-WCP+}
The correction in \eqref{eq:rho_trans} is worst-case, as it is derived without imposing any restriction on how the observational and interventional nonconformity scores of the affected units may differ. Appendix \ref{sec:trans_equal_coverage} provides an example in which the correction is tight in terms of marginal coverage, in the sense that the prediction sets induced by the ideal interventional and corrected observational $p$-values have identical marginal coverage. In many applications, however, changing the treatment of one unit induces only a bounded perturbation in the nonconformity scores of the affected units. We denote a bound on the perturbation in the score of unit $j$ caused by changing the treatment of unit $I$ by the scalar $\Delta_j^I$, which is formally defined in Appendix \ref{app:ia_wcp_plus_trans}. Under the technical assumptions stated therein, this bound yields a sharper correction that accounts for the magnitude of the score perturbation
\begin{align}
	\rho_{I,y}^{+}=\frac{\sum_{j\in\mathcal H_I} W_j \mathds 1 \left\{V_I(y)-\Delta_j^I \le V_j^{\mathrm{obs}} < V_I(y) \right\}}{W_I+\sum_{j\in\mathcal I_1}W_j}.
	\label{eq:rho_delta_trans_plus}
\end{align}
Replacing \eqref{eq:rho_trans} with \eqref{eq:rho_delta_trans_plus} yields a stability-aware variant of IA-WCP, which we call IA-WCP+. As shown in Appendix \ref{app:ia_wcp_plus_trans}, IA-WCP+ preserves the coverage requirement \eqref{eq:cov_trans_iawcp} under the mentioned technical assumptions.

\section{Experiments}
\label{sec:experiments}
We evaluate the proposed methods using a synthetic peer-interference model inspired by \citep{forastiere2021identification,chin2019regression}. Appendix \ref{app:add_experiments} provides additional results for this setting, along with experiments in a wireless traffic-slicing problem \citep{foukas2017network,bao2017prediction}.

\subsection{Setup}
We consider a network of $N=300$ units in which each unit $i$ has covariate $X_i=(U_i,G_i)$, where $U_i\sim\operatorname{Unif}[-1,1]$ and $G_i\sim\operatorname{Unif}\{1,\ldots,G\}$ is a group label. Treatments are drawn independently according to $T_i\mid X_i\sim\operatorname{Bernoulli}(\sigma(U_i/2))$, and the neighborhood of unit $i$ contains its same-group peers, $\mathcal N_i(X_{1:N})=\{j\ne i:G_j=G_i\}$. The exposure is given by the fraction of treated units in the same group, i.e., 
\begin{align}
 E_i=\frac{\sum_{j\in\mathcal N_i(X_{1:N})}T_j}{ |\mathcal N_i(X_{1:N})|},
 \end{align}
 with the convention that $E_i=0$ if $\mathcal N_i(X_{1:N})=\emptyset$. The potential outcomes are 
 \begin{align}
 Y_i(0)&=0,\\
 Y_i(1)&=\theta U_i+\left[\sigma(\lambda(E_i-0.5))-\sigma(-0.5 \lambda )\right] +\varepsilon_i,
 \label{eq:outcome_ablation}
\end{align}
where $\sigma(\cdot)$ denotes the sigmoid function and $\varepsilon_i\sim\mathcal N(0,\sigma_\varepsilon^2)$. This model captures a simple peer-effect mechanism in which each unit responds to the treatment rate in its group, with the response becoming stronger once peer adoption crosses 50\%. Such a mechanism can represent settings with social reinforcement or congestion effects. Unless specified otherwise, we instantiate the model with $G = 16$ groups, linear coefficient $\theta = 0.25$, sigmoid parameter $\lambda = 10$, and noise standard deviation $\sigma_\varepsilon = 0.05$. 

\subsection{Benchmarks}
\label{subsec:exp_benchmarks}
We use the absolute-residual score and compare WCP \citep{lei2021conformal}, ideal WCP based on the unobservable interventional $p$-value in \eqref{eq:conf_p_value_trans_int}, and the proposed IA-WCP and IA-WCP+. Ideal WCP serves as an oracle reference, and it is available here because the data-generating mechanism is known. We apply these calibration methods to three different base predictors. The first is the interference-free linear predictor $\hat Y(1)=\theta U$, which returns the mean outcome in the absence of interference. We also consider two fitted linear predictors, a covariate-only predictor  $\hat Y(1)=\hat\theta_U U+\hat \theta_0$, and an exposure-aware predictor $\hat Y(1)=\hat\theta_E  E+\hat\theta_U U+\hat \theta_0$. The two linear predictors are fitted on treated units from ten independent training networks.

For IA-WCP+, the stability margin is given by
\begin{align}
\Delta_j^I=\left(\frac{\lambda}{4}+|\hat\theta_E|\right)\left|\tilde E_j^I-E_j\right|,
\label{eq:synthetic_stability_margin}
\end{align}
where $\hat\theta_E$ is set to zero for predictors that do not use the exposure level $E$. 
\subsection{Results}

We first study the effect of neighborhood size by varying the number of groups $G$. As shown in Figure \ref{fig:main_synthetic_groups}, smaller values of $G$ produce larger neighborhoods and increase the fraction of calibration units affected by the intervention. In this regime, WCP undercovers with the miscoverage gap increasing as the number of groups $G$ decreases. For example, at $G=4$, its empirical coverage is approximately $0.75$, below the target level $0.8$. Across all values of $G$, ideal WCP and the proposed IA-WCP and IA-WCP+ attain or exceed the desired coverage level. However, because IA-WCP uses a worst-case correction, it becomes highly conservative for small values of $G$. In contrast, IA-WCP+ performance tracks that of ideal WCP and is substantially more efficient than IA-WCP.

\begin{figure}[t]
    \centering
    \includegraphics[width=0.475\textwidth]{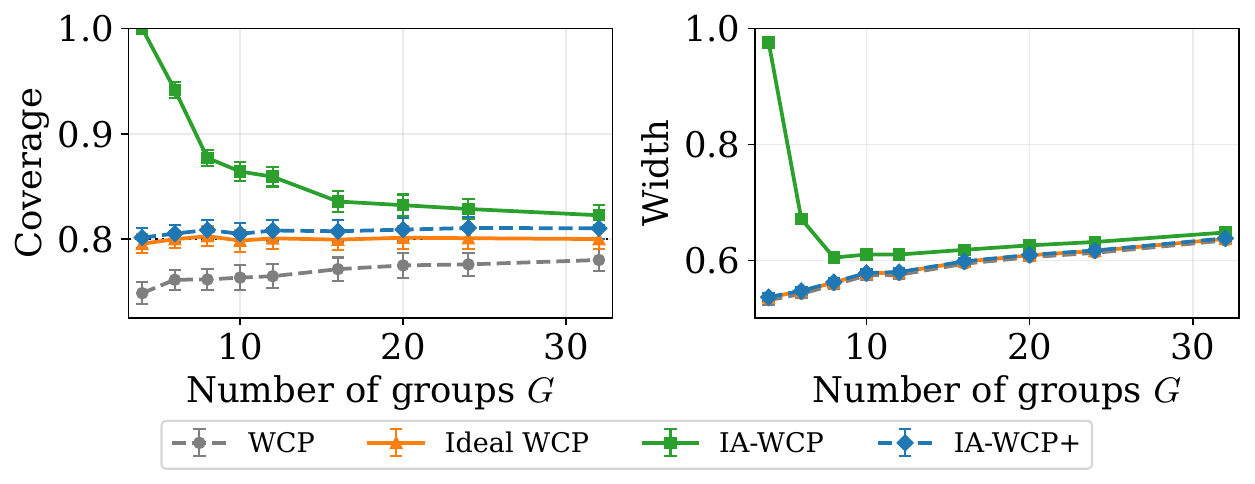}
    \caption{Empirical coverage and normalized prediction-set width of WCP, ideal WCP, IA-WCP, and IA-WCP+ as a function of the number of groups $G$. The target coverage is set to $1-\alpha=0.8$. Results are averaged over 200 runs.}
    \label{fig:main_synthetic_groups}
\end{figure}

We next investigate how the choice of base predictor affects coverage and efficiency. Figure \ref{fig:main_synthetic_predictors} shows that replacing the interference-free linear predictor with predictors fitted on data generated under interference substantially reduces prediction-set width. It also reduces the empirical miscoverage gap of WCP, whose coverage approaches that of ideal WCP when the predictor accounts for the exposure level. It is important to highlight that this behavior is specific to the present experiment, as conventional WCP lacks a general coverage guarantee under interference. The proposed methods retain their validity across all three predictors. Moreover, IA-WCP+ is noticeably narrower than IA-WCP when using the interference-free linear predictor or covariate-only predictor. When exposure $E$ is included in the predictor, the additional term $|\hat\theta_E|$ enlarges the stability margin \eqref{eq:synthetic_stability_margin}, causing IA-WCP+ to nearly coincide with IA-WCP.

\begin{figure}[t]  
    \centering
    \includegraphics[width=0.475\textwidth]{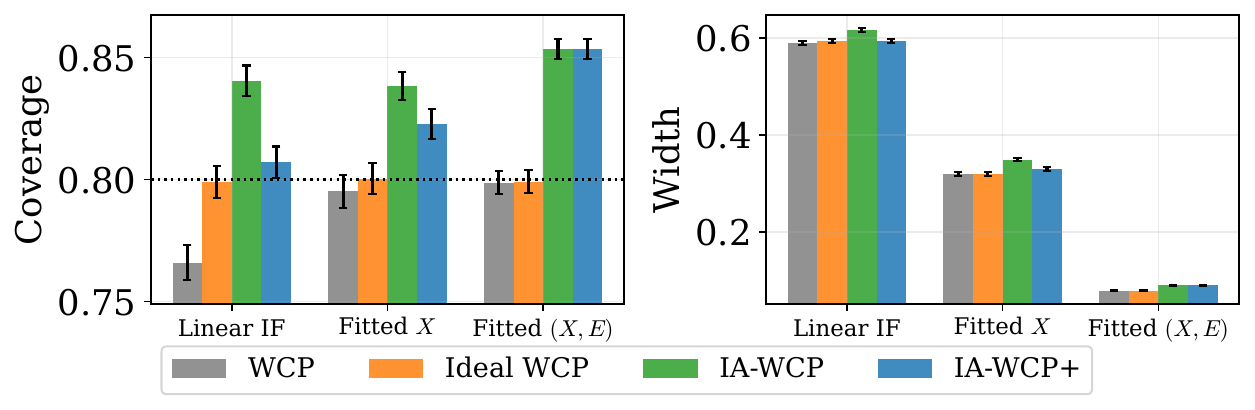}
    \caption{Empirical coverage and normalized prediction-set width of WCP, ideal WCP, IA-WCP, and IA-WCP+ when applied to the interference-free linear predictor (IF), the covariate-only predictor (Fitted X), and the exposure-aware predictor (Fitted (X,E)). The target coverage is set to $1-\alpha=0.8$. Results are averaged over 200 runs.}
    \label{fig:main_synthetic_predictors}
\end{figure}

\section{Conclusion}

We studied conformal counterfactual prediction for individual potential outcomes and treatment effects under interference in the inductive and transductive settings. In either setting, a unit's potential outcomes may depend on the covariates and treatment assignments of other units, creating a source of non-exchangeability that is not corrected by standard unit-level propensity weighting and may invalidate existing conformal methods.

To address this challenge, we have developed two interference-adjusted calibration methods. The first, IA-WCP, retains finite-sample marginal coverage under interference by adjusting the observational conformal $p$-value using a worst-case bound on the discrepancy between observational and interventional calibration scores. Although this correction remains valid under arbitrary score perturbations, it can produce conservative prediction sets. When bounds on intervention-induced score perturbations are available, IA-WCP+ uses them to sharpen the correction and improve efficiency while retaining the same coverage guarantee. Our experiments show that conventional WCP can remain close to nominal coverage when interference is weak, but it may undercover as interference becomes stronger. In contrast, IA-WCP and IA-WCP+ maintain the target coverage, with IA-WCP+ generally producing smaller prediction sets.

\bibliographystyle{apalike}
\bibliography{references}

\clearpage
\appendix
\onecolumn
\aistatstitle{Conformal Individual Treatment Effect Estimation under Networked Interference: Supplementary Materials}

\section{Related Work}
\label{app:related_work}
\paragraph{Conformal Inference for Treatment Effects.}
Conformal prediction (CP) is a calibration framework that transforms point predictions into prediction sets with finite-sample marginal coverage guarantees, provided that the calibration and test data are exchangeable \citep{shafer2008tutorial}. In counterfactual inference, however, exchangeability is typically violated because outcomes under a given treatment are observed only for units assigned to that treatment, inducing covariate shift between the calibration and test distributions. WCP addresses this shift by reweighting the calibration data \citep{tibshirani2019conformal}. Under the Stable Unit Treatment Value Assumption (SUTVA), unconfoundedness, and overlap, WCP constructs valid prediction sets for counterfactual outcomes and individual treatment effects (ITEs) under covariate-dependent treatment assignment \citep{lei2021conformal}.

Subsequent work has relaxed these assumptions in several directions. Sensitivity analysis has been introduced to address potential violations of unconfoundedness \citep{jin2023sensitivity,yin2024conformal}; interventional data have been leveraged to mitigate hidden confounding \citep{chen2024conformal,qchohi2026confounding}; and synthetic-powered prediction has been proposed to reduce the inefficiency caused by poor overlap \citep{farzaneh2025synthetic}. However, all these methods fundamentally rely on the no-interference assumption embedded in SUTVA. In this work, we relax this assumption and construct valid conformal prediction sets for counterfactual outcomes and ITEs under network interference.

\paragraph{Causal inference under interference.}
The causal-inference literature relaxes the no-interference component of SUTVA by imposing structure on how one unit's treatment may affect another unit's outcome. Early work studies direct and spillover effects under partial interference, where interference is restricted to known groups \citep{hudgens2008toward,tchetgen2012causal}. Exposure mappings extend this framework to richer interference patterns by summarizing the treatment assignments of other units through an exposure variable, thereby enabling the identification and estimation of causal effects under randomized or observational treatment assignment \citep{aronow2017,forastiere2021identification}. More recent machine-learning approaches exploit network structure to estimate heterogeneous treatment effects, exposure-response functions, and treatment policies in networked populations \citep{ma2021causal,huang2024modeling}. These works have primarily focused on defining causal estimands and estimating population-level or heterogeneous effects under interference. In contrast, we focus on constructing model-agnostic, finite-sample prediction sets for individual potential outcomes and treatment effects. 

\paragraph{Conformal Prediction on Networked Data.}
A growing body of literature adapts conformal prediction to graph-structured observations. For node-level prediction, existing methods use graph neighborhoods to reweight nonconformity scores \citep{clarkson2023distribution, zargarbashi2023conformal}, or learn topology-aware corrections that exploit permutation invariance to preserve marginal validity \citep{huang2023uncertainty}. Related work addresses the distribution shift induced when new nodes or edges enter an inductive graph \citep{h2024conformal}, and extends conformal uncertainty quantification to link prediction \citep{zhao2024conformalized}. In these settings, the graph structure is leveraged to learn node representations and improve calibration efficiency. However, the target remains an unobserved label or edge in the factual graph, rather than a counterfactual outcome under a treatment intervention. The closest work to ours is that of \cite{zhou2025adaptive}, which exploits a fixed transductive network to learn representations that account for hidden confounding, adapts conformal calibration to node degree, and estimates ITEs. Nevertheless, the framework in \citep{zhou2025adaptive} explicitly assumes no network interference, whereas in our setting, treatment spillover across neighboring nodes directly alters the target counterfactual outcomes.

\section{Examples of Permutation-equivariant Neighborhood and Exposure Maps}
\label{sec:perm_eq_maps}

We now list several permutation-equivariant neighborhood rules and deterministic exposure maps satisfying Assumption \ref{ass:eq_exp_map}. These maps induce degenerate exposure kernels of the form
\begin{align}
    P_{E\mid X,  \mathcal{L}}\left(\cdot\mid X_i,  \mathcal{L}_i(X_{1:N},T_{1:N})\right)=\delta_{e_i(X_{1:N},T^{-i}_{1:N})}(\cdot),
\end{align}
where $e_i(X_{1:N},T^{-i}_{1:N})$ denotes the exposure assigned to unit $i$, and $\delta_x(\cdot)$ denotes the Dirac measure centered at $x$. 

Throughout this section we use the same permutation convention as in Assumption \ref{ass:eq_exp_map}. For a permutation $\sigma$ of $\{1,\ldots,N\}$, we denote the permuted treatment and covariate vectors by
\begin{align}
	T_{\sigma(1:N)}=(T_{\sigma(1)},\ldots,T_{\sigma(N)}), \qquad X_{\sigma(1:N)}=(X_{\sigma(1)},\ldots,X_{\sigma(N)}).
\end{align}
Thus, under the relabeled vectors $(T_{\sigma(1:N)},X_{\sigma(1:N)})$, the unit at position $i$ corresponds to the original unit $\sigma(i)$. A deterministic exposure map is permutation-equivariant if, for every permutation $\sigma$ and every $i=1,\ldots,N$,
\begin{align}
	e_i(X_{\sigma(1:N)},T^{-i}_{\sigma(1:N)}) = e_{\sigma(i)}(X_{1:N},T^{-\sigma(i)}_{1:N}).
\end{align}

\begin{enumerate}
	\item \textbf{Global treated-count exposure.}
	Set $\mathcal N_i(X_{1:N})=\{1,\ldots,N\}\setminus\{i\}$ and define
	\begin{align}
		e_i(X_{1:N},T^{-i}_{1:N}) = \sum_{j\neq i}T_j.
	\end{align}
	Then, for any permutation $\sigma$,
	\begin{align}
		e_i(X_{\sigma(1:N)},T^{-i}_{\sigma(1:N)}) =  \sum_{\ell\neq i}T_{\sigma(\ell)} = \sum_{j\neq \sigma(i)}T_j = e_{\sigma(i)}(X_{1:N},T^{-\sigma(i)}_{1:N}).
	\end{align}

	\item \textbf{Fraction treated outside unit $i$.}
	Again set $\mathcal N_i(X_{1:N})=\{1,\ldots,N\}\setminus\{i\}$ and define
	\begin{align}
		e_i(X_{1:N},T^{-i}_{1:N}) = \frac{1}{N-1}\sum_{j\neq i}T_j.
	\end{align}
	Then, for any permutation $\sigma$,
	\begin{align}
		e_i(X_{\sigma(1:N)},T^{-i}_{\sigma(1:N)})= \frac{1}{N-1}\sum_{\ell\neq i}T_{\sigma(\ell)}= \frac{1}{N-1}\sum_{j\neq \sigma(i)}T_j = e_{\sigma(i)}(X_{1:N},T^{-\sigma(i)}_{1:N}).
	\end{align}

	\item \textbf{$k$-nearest-neighbor treated-count exposure.}
	Let $d:\mathcal X\times\mathcal X\to\mathbb R_+$ be a distance, and let $\mathcal N_i^{(k)}(X_{1:N})$ denote the set of indices corresponding to the $k$ nearest neighbors of unit $i$ among $\{1,\ldots,N\}\setminus\{i\}$. We assume either that ties occur with probability zero, or that they are resolved by a permutation-equivariant rule. Define
	\begin{align}
		e_i(X_{1:N},T^{-i}_{1:N}) = \sum_{j\in\mathcal N_i^{(k)}(X_{1:N})}T_j.
	\end{align}
	Under the permutation convention above, the nearest-neighbor set satisfies
	\begin{align}
		\mathcal N_i^{(k)}(X_{\sigma(1:N)}) = \{\sigma^{-1}(j):j\in\mathcal N_{\sigma(i)}^{(k)}(X_{1:N})\}.
	\end{align}
	Therefore,
	\begin{align}
		e_i(X_{\sigma(1:N)},T^{-i}_{\sigma(1:N)}) = \sum_{\ell\in\mathcal N_i^{(k)}(X_{\sigma(1:N)})} \hspace{-2em}T_{\sigma(\ell)} = \sum_{j\in\mathcal N_{\sigma(i)}^{(k)}(X_{1:N})} \hspace{-1.5em} T_j= e_{\sigma(i)}(X_{1:N},T^{-\sigma(i)}_{1:N}).
	\end{align}

	\item \textbf{Radius-based treated-count exposure.}
	Let $d:\mathcal X\times\mathcal X\to\mathbb R_+$ be a distance and fix $r>0$. Set
	\begin{align}
		\mathcal N_i(X_{1:N}) = \{j\ne i:d(X_i,X_j)\le r\},
	\end{align}
	and define
	\begin{align}
		e_i(X_{1:N},T^{-i}_{1:N}) = \sum_{j\in \mathcal N_i(X_{1:N})} T_j.
	\end{align}
	Then, for any permutation $\sigma$,
	\begin{align}
		e_i(X_{\sigma(1:N)},T^{-i}_{\sigma(1:N)})= \sum_{\ell\neq i} T_{\sigma(\ell)} \mathds 1\{d(X_{\sigma(i)},X_{\sigma(\ell)})\le r\} &= \sum_{j\neq \sigma(i)} T_j \mathds 1\{d(X_{\sigma(i)},X_j)\le r\} \nonumber\\
		&= e_{\sigma(i)}(X_{1:N},T^{-\sigma(i)}_{1:N}).
	\end{align}

	\item \textbf{Weighted exposure.}
	Let $a:\mathcal X\times\mathcal X\to\mathbb R$ be a measurable weight function. Set $\mathcal N_i(X_{1:N})=\{1,\ldots,N\}\setminus\{i\}$ and define
	\begin{align}
		e_i(X_{1:N},T^{-i}_{1:N}) = \sum_{j\neq i}a(X_i,X_j)T_j.
	\end{align}
	Then, for any permutation $\sigma$,
	\begin{align}
		e_i(X_{\sigma(1:N)},T^{-i}_{\sigma(1:N)}) = \hspace{-0.2em}\sum_{\ell\neq i} a(X_{\sigma(i)},X_{\sigma(\ell)})T_{\sigma(\ell)} &= \hspace{-0.2em}\sum_{j\neq \sigma(i)} a(X_{\sigma(i)},X_j)T_j = e_{\sigma(i)}(X_{1:N},T^{-\sigma(i)}_{1:N}).
	\end{align}

	\item \textbf{Threshold exposure.}
	Let $m\in\{1,\ldots,N-1\}$, set $\mathcal N_i(X_{1:N})=\{1,\ldots,N\}\setminus\{i\}$, and define
	\begin{align}
		e_i(X_{1:N},T^{-i}_{1:N}) = \mathds 1 \left\{ \sum_{j\neq i}T_j\ge m \right\}.
	\end{align}
	Since the treated count outside unit $i$ is permutation-equivariant,
	\begin{align}
		e_i(X_{\sigma(1:N)},T^{-i}_{\sigma(1:N)}) = \mathds 1 \left\{ \sum_{\ell\neq i}T_{\sigma(\ell)}\ge m \right\}= \mathds 1 \left\{ \sum_{j\neq \sigma(i)}T_j\ge m \right\} = e_{\sigma(i)}(X_{1:N},T^{-\sigma(i)}_{1:N}).
	\end{align}

\end{enumerate}

\section{Proof of Transductive IA-WCP Validity}
\label{app:transductive_validity}

We prove the result for the case in which the test unit $I$ satisfies $T_{I}=0$, so that the missing potential outcome is $Y_{I}(1)$. The case $T_{I}=1$ is analogous. All probabilities in the remainder of this proof are conditional on $T_I=0$.

We first establish weighted exchangeability of the interventional sample and the validity of the ideal WCP procedure under interference. We then compare the observational and interventional $p$-values under a joint coupling and conclude the validity proof of IA-WCP.

\subsection{Validity of the ideal WCP procedure}
\label{app:transductive_weighted_exchangeability}

Let $\tilde T^{I}_{1:N}$ denote the treatment vector obtained by setting the treatment of unit $I$ to one while leaving all other treatment assignments in $T_{1:N}$ unchanged, i.e.,
\begin{align}
	\tilde T^{I}_{I}=1, \qquad \tilde T^{I}_j=T_j,\quad j\neq I.
\end{align}
Under this intervention, let $\tilde E_j^{I}$ denote the exposure of unit $j$, and let $\tilde Y_j^{I}(1)$ denote the corresponding treated potential outcome. Define
\begin{align}
	\tilde Z_j = (X_j,\tilde E_j^{I},\tilde Y_j^{I}(1)), \qquad j=1,\dots,N.
\end{align}
We now show that, for any treatment vector $T_{1:N}$ such that $T_{I}=0$, the augmented interventional sample
\begin{align}
	\tilde{\mathcal D}_1 = (\tilde Z_j)_{j\in\mathcal I_1\cup\{I\}},
\end{align}
is weighted exchangeable.

\begin{lemma}[Weighted exchangeability of the interventional sample]
\label{lemma:w_exch:merged}
Fix an index $I$ and a treated calibration set $\mathcal I_1\subseteq \{1,\dots,N\}\setminus\{I\}$, and let $\mathcal{S}=\mathcal I_1\cup\{I\}$. Work conditionally on the observational assignment event
\begin{align}
	A_{\mathcal I_1,I} = \{T_j=1 \text{ for } j\in\mathcal I_1, T_{I}=0, T_j=0 \text{ for } j\notin \mathcal{S}\}.
\end{align}
The intervention is applied after conditioning on this observational assignment event.

Under Assumptions \ref{ass:prop_unit}--\ref{ass:consistency}, the conditional density of $(\tilde Z_j)_{j\in \mathcal{S}}$ admits the decomposition
\begin{align}
	p\bigl((\tilde Z_j)_{j\in \mathcal{S}}\mid A_{\mathcal I_1,I}\bigr) = cw(X_{I}) g\bigl((\tilde Z_j)_{j\in \mathcal{S}}\bigr),
\end{align}
where $c>0$,
\begin{align}
	w(x)=\frac{1-\pi(x)}{\pi(x)},
\end{align}
and $g$ is invariant under permutations of the arguments indexed by $\mathcal{S}$.
\end{lemma}
\begin{proof}
By Assumption \ref{ass:prop_unit},
\begin{align}
    \Pr(A_{\mathcal I_1,I}\mid X_{1:N}) = \prod_{j\in\mathcal I_1}\pi(X_j) \bigl(1-\pi(X_I)\bigr) \prod_{j\notin\mathcal S}\bigl(1-\pi(X_j)\bigr).
\end{align}
Therefore,
\begin{align}
	p(X_{1:N}\mid A_{\mathcal I_1,I}) &= \frac{ \Pr(A_{\mathcal I_1,I}\mid X_{1:N})p_X^{\otimes N}(X_{1:N}) }{ \Pr(A_{\mathcal I_1,I}) } \nonumber\\
	&= c_0 \prod_{j\in\mathcal I_1}\pi(X_j)p_X(X_j) \bigl(1-\pi(X_{I})\bigr)p_X(X_{I}) \prod_{j\notin \mathcal{S}}\bigl(1-\pi(X_j)\bigr)p_X(X_j),
\end{align}
where $c_0=1/\Pr(A_{\mathcal I_1,I})$. Multiplying and dividing the factor associated with
$I$ by $\pi(X_{I})$ gives
\begin{align}
	p(X_{1:N}\mid A_{\mathcal I_1,I}) = c_0 w(X_{I}) \prod_{j\in \mathcal{S}}\pi(X_j)p_X(X_j) \prod_{j\notin \mathcal{S}}\bigl(1-\pi(X_j)\bigr)p_X(X_j).
	\label{eq:cond_cov_factorization:merged}
\end{align}

Integrating out the covariates indexed by $\mathcal S^c$ contributes only a constant independent of $(X_j)_{j\in\mathcal S}$. The resulting marginal density is therefore proportional to
\begin{align}
    w(X_I)\prod_{j\in\mathcal S}\pi(X_j)p_X(X_j),
\end{align}
which proves the weighted exchangeability of covariates conditioned on the event $A_{\mathcal I_1,I}$.

We now define the interventional variables using the intervened treatment vector $\tilde T^{I}$. Conditionally on $(X_{1:N},\tilde T^{I})$, Assumption \ref{ass:exposure_mechanism} gives
\begin{align}
	p(\tilde E_{1:N}^{I}\mid X_{1:N},A_{\mathcal I_1,I},\operatorname{do}(T_I=1)) = \prod_{j=1}^N p_{E\mid X,  \mathcal{L}} \left( \tilde E_j^{I} \mid X_j,  \mathcal{L}_j(X_{1:N},\tilde T^I_{1:N}) \right).
	\label{eq:interventional_exposure_factorization:merged}
\end{align}
Similarly, by Assumption \ref{ass:no_hidden_interference_confounding},
\begin{align}
	p\left( \tilde Y_{1:N}^{I}(1) \mid X_{1:N},\tilde E_{1:N}^{I},A_{\mathcal I_1,I},\operatorname{do}(T_I=1) \right) = \prod_{j=1}^N p_{Y(1)\mid X,E} \left( \tilde Y_j^{I}(1) \mid X_j,\tilde E_j^{I} \right).
	\label{eq:interventional_outcome_factorization:merged}
\end{align}

Combining \eqref{eq:cond_cov_factorization:merged}, \eqref{eq:interventional_exposure_factorization:merged}, and \eqref{eq:interventional_outcome_factorization:merged}, and then marginalizing over all variables indexed by $\mathcal{S}^c=\{1,\dots,N\}\setminus \mathcal{S}$, yields
\begin{align}
	p\bigl((\tilde Z_j)_{j\in \mathcal{S}}\mid A_{\mathcal I_1,I}\bigr) = cw(X_{I}) g\bigl((\tilde Z_j)_{j\in \mathcal{S}}\bigr),
\end{align}
for a normalizing constant $c>0$, where $g$ collects all remaining factors after the extraction of $w(X_{I})$.

It remains to show that $g$ is invariant under permutations of its arguments. Let $\sigma$ be any permutation of the indices in $\mathcal{S}$, extended to $\{1,\dots,N\}$ by setting $\sigma(k)=k$ for $k\notin \mathcal{S}$. The factors
\begin{align}
	\prod_{j\in \mathcal{S}}\pi(X_j)p_X(X_j)
\end{align}
and
\begin{align}
	\prod_{j\in \mathcal{S}} p_{Y(1)\mid X,E} \left( \tilde Y_j^{I}(1) \mid X_j,\tilde E_j^{I} \right)
\end{align}
are invariant under permutations of the triples
\begin{align}
	(X_j,\tilde E_j^{I},\tilde Y_j^{I}(1)), \qquad j\in \mathcal{S}.
\end{align}

Finally, under $\tilde T^{I}$, all units in $\mathcal{S}$ have treatment equal to one. Hence a permutation of the labels in $\mathcal{S}$ only relabels units with the same intervened treatment. By Assumption \ref{ass:eq_exp_map}, the neighborhood rule is permutation equivariant and the exposure kernel is invariant to the ordering of the neighbors, so the joint exposure density is unchanged under the same simultaneous permutation of the covariates, exposures, and outcomes indexed by $\mathcal{S}$. The variables indexed by $\mathcal{S}^c$ are integrated out, and the integration measure is invariant under this relabeling.

Therefore $g$ is invariant under permutations of the arguments $(\tilde Z_j)_{j\in \mathcal{S}}$, proving the claimed decomposition.
\end{proof}

By Lemma \ref{lemma:w_exch:merged}, conditionally on $A_{\mathcal I_1,I}$, the augmented interventional sample is weighted exchangeable. Therefore, the standard weighted conformal validity argument \citep{tibshirani2019conformal} gives
\begin{align}
    \Pr\left[ p_{I,\tilde Y_I^{I}(1)}^{\mathrm{int}} \le \alpha \middle| A_{\mathcal I_1,I} \right] \le \alpha.
    \label{eq:interventional_superuniform_conditional:merged}
\end{align}

This establishes conditional validity at the interventional target outcome $\widetilde Y_I^I(1)$. Because $I\notin\mathcal N_I(X_{1:N})$, changing $T_I$ does not alter the exposure law of the target unit. Under Assumptions \ref{ass:exposure_mechanism} and \ref{ass:no_hidden_interference_confounding}, conditionally on $A_{\mathcal I_1,I}$, we may therefore couple the target variables so that
\begin{align}
    (\widetilde E_I^I,\widetilde Y_I^I(1))
    =
    (E_I,Y_I(1))
    \qquad\text{a.s.},
\end{align}
while preserving their joint law with the interventional calibration scores. Consequently,
\begin{align}
    \Pr\left[
        p_{I,Y_I(1)}^{\mathrm{int}}\le\alpha
        \,\middle|\,
        A_{\mathcal I_1,I}
    \right]
    \le \alpha.
\end{align}
For each fixed $I$, the events $A_{\mathcal I_1,I}$, indexed by $\mathcal I_1\subseteq\{1,\ldots,N\}\setminus\{I\}$, form a partition of $\{T_I=0\}$. Applying the law of total probability over this partition and then averaging over the selected index $I$ proves Lemma \ref{thm:ideal_trans_cov}.

\subsection{Coupling comparison}
\label{app:transductive_coupling}

For this proof, we consider the exposure-law affected set
\begin{align}
    \mathcal H_I^{\mathrm{ex}} = \Bigl\{j\in\mathcal I_1: &P_{E\mid X,  \mathcal{L}} \left(\cdot\mid X_j,  \mathcal{L}_j(X_{1:N},T_{1:N})\right) \neq P_{E\mid X,  \mathcal{L}} \left(\cdot\mid X_j,  \mathcal{L}_j(X_{1:N},\tilde T^I_{1:N})\right) \Bigr\},
    \label{eq:exact_affected_trans:merged}
\end{align}
and write $\mathcal U_I^{\mathrm{ex}}=\mathcal I_1\setminus\mathcal H_I^{\mathrm{ex}}$. Note that $\mathcal H_I^{\mathrm{ex}}\subseteq\mathcal H_I$ because the conditional exposure law of unit $j$ is unchanged when $I\notin\mathcal N_j(X_{1:N})$. Thus, proving the comparison first with the set $\mathcal H_I^{\mathrm{ex}}$ yields the main-text bound by enlarging the index set of the nonnegative correction sum from $\mathcal H_I^{\mathrm{ex}}$ to $\mathcal H_I$.

\begin{lemma}
	\label{lem:trans_p_value_comparison:merged}
	Under Assumptions \ref{ass:exposure_mechanism} and \ref{ass:no_hidden_interference_confounding}, conditionally on $T_I=0$, the observational conformal $p$-value in \eqref{eq:conf_p_value_trans} and the interventional conformal $p$-value in \eqref{eq:conf_p_value_trans_int} admit a coupling such that, simultaneously for every $y\in\mathcal Y$,
	\begin{align}
		p_{I,y}^{\mathrm{int}} \le p_{I,y}^{\mathrm{obs}}+\rho^{\rm ex}_{I,y} \qquad\text{a.s.},
	\end{align}
	where 
	\begin{align}
		\rho^{\rm ex}_{I,y} = \frac{ \sum_{j\in\mathcal H_I^{\mathrm{ex}}} W_j\mathds 1\{V_j^{\mathrm{obs}}<V_I(y)\} }{ W_I+\sum_{j\in\mathcal I_1}W_j }.
	\end{align}
\end{lemma}

\begin{proof}
	We construct the coupling conditionally on $(X_{1:N},T_{1:N},I)$, with $T_I=0$. Since the exposure of a unit does not depend on its own treatment,
	\begin{align}
		E_I\mid X_{1:N},T_{1:N} \overset{d}{=} \tilde E_I^{I}\mid X_{1:N},\tilde T^I_{1:N},
	\end{align}
	where $\overset{d}{=}$ denotes equality in distribution.
	Moreover, for every $j\in\mathcal U_I^{\mathrm{ex}}$, the definition of $\mathcal U_I^{\mathrm{ex}}$ gives
	\begin{align}
		E_j\mid X_{1:N},T_{1:N} \overset{d}{=} \tilde E_j^{I}\mid X_{1:N},\tilde T^I_{1:N}.
	\end{align}
	We may therefore couple these exposures so that
	\begin{align}
		E_I = \tilde E_I^{I}, \quad \text{ and } \quad E_j = \tilde E_j^{I}, \qquad j\in\mathcal U_I^{\mathrm{ex}},
	\end{align}
	almost surely.
	 Conditional on these coupled exposures, the corresponding treated potential outcomes have the same conditional distribution $P_{Y(1)\mid X,E}$. Hence, they may be coupled so that 
	\begin{align}
		Y_I(1) = \tilde Y_I^{I}(1), \quad \text{ and } \quad Y_j(1) = \tilde Y_j^{I}(1), \qquad j\in\mathcal U_I^{\mathrm{ex}},
	\end{align}
	almost surely.  For every $j\in\mathcal H_I^{\mathrm{ex}}$, choose any coupling with the correct observational and interventional marginals.

	By Assumptions \ref{ass:exposure_mechanism} and \ref{ass:no_hidden_interference_confounding}, the unit-specific exposure and outcome variables are conditionally independent across units. These couplings can therefore be combined into a product coupling that preserves the observational and interventional score-vector marginals and their joint laws with the corresponding target potential outcome. Under this coupling, for $y\in\mathcal Y$,
	\begin{align}
		V_I^{\mathrm{int}}(y) = V_I(y), \quad \text{ and } \quad V_j^{\mathrm{int}} = V_j^{\mathrm{obs}}, \qquad j\in\mathcal U_I^{\mathrm{ex}},
	\end{align}
	almost surely. The conformal weights are also unchanged because the intervention does not alter the covariates.

	Consequently, the unaffected units cancel from the difference between the two $p$-values and
	\begin{align}
		p_{I,y}^{\mathrm{int}}-p_{I,y}^{\mathrm{obs}} &= \frac{\sum_{j\in\mathcal H_I^{\mathrm{ex}}} W_j \left[ \mathds 1\{V_j^{\mathrm{int}}\ge V_I(y)\} - \mathds 1\{V_j^{\mathrm{obs}}\ge V_I(y)\} \right] }{ W_I+\sum_{j\in\mathcal I_1}W_j }.
	\end{align}
	Note that for every $j\in\mathcal H_I^{\mathrm{ex}}$,
	\begin{align}
		\mathds 1\{V_j^{\mathrm{int}}\ge V_I(y)\} - \mathds 1\{V_j^{\mathrm{obs}}\ge V_I(y)\} \le \mathds 1\{V_j^{\mathrm{obs}}<V_I(y)\},
	\end{align}
	and therefore,
	\begin{align}
		p_{I,y}^{\mathrm{int}}-p_{I,y}^{\mathrm{obs}} \le \frac{ \sum_{j\in\mathcal H_I^{\mathrm{ex}}} W_j\mathds 1\{V_j^{\mathrm{obs}}<V_I(y)\} }{ W_I+\sum_{j\in\mathcal I_1}W_j } \le \frac{ \sum_{j\in\mathcal H_I} W_j\mathds 1\{V_j^{\mathrm{obs}}<V_I(y)\} }{ W_I+\sum_{j\in\mathcal I_1}W_j } =\rho_{I,y}.
	\end{align}
	Hence,
	\begin{align}
		p_{I,y}^{\mathrm{int}}\le p_{I,y}^{\mathrm{obs}}+\rho^{\rm ex}_{I,y} \le p_{I,y}^{\mathrm{obs}}+\rho_{I,y} \qquad\text{a.s.}
	\end{align}
	Since the coupling does not depend on $y$, this inequality holds simultaneously for every $y\in\mathcal Y$.
\end{proof}

\subsection{Coverage of IA-WCP}
\label{app:transductive_coverage}

Use the joint coupling constructed in Lemma \ref{lem:trans_p_value_comparison:merged}. It preserves the joint law of the interventional scores and target outcome and satisfies
\begin{align}
    Y_I(1)=\tilde Y_I^I(1), \qquad
    p_{I,Y_I(1)}^{\mathrm{int}}=p_{I,\tilde Y_I^I(1)}^{\mathrm{int}}
\end{align}
almost surely.

For each fixed $I$, the events $A_{\mathcal I_1,I}$, indexed by the possible treated sets $\mathcal I_1\subseteq\{1,\ldots,N\}\setminus\{I\}$, partition the event $\{T_I=0\}$. Averaging \eqref{eq:interventional_superuniform_conditional:merged} over these events and over the randomly selected index $I$, and using the target equality under the coupling, yields
\begin{align}
    \Pr\left[ p_{I,Y_I(1)}^{\mathrm{int}} \le \alpha \middle| T_I=0 \right] \le \alpha.
    \label{eq:interventional_superuniform:merged}
\end{align}

By Lemma \ref{lem:trans_p_value_comparison:merged}, under the same coupling, almost surely and simultaneously for every $y\in\mathcal Y$,
\begin{align}
	p_{I,y}^{\mathrm{int}} \le p_{I,y}^{\mathrm{obs}}+\rho_{I,y}.
\end{align}
Since overlap ensures $p_{I,y}^{\mathrm{obs}}>0$, rejection is impossible when $\rho_{I,y}\ge\alpha$. Otherwise, rejection implies $p_{I,y}^{\mathrm{obs}}+\rho_{I,y}\le\alpha$. Therefore,
\begin{align}
	\left\{ p_{I,y}^{\mathrm{obs}} \le (\alpha-\rho_{I,y})_+ \right\} \subseteq \left\{ p_{I,y}^{\mathrm{int}} \le \alpha \right\}.
\end{align}
Applying this inclusion at $y=Y_{I}(1)$ gives
\begin{align}
	\Pr \left[ p_{I,Y_{I}(1)}^{\mathrm{obs}} \le (\alpha-\rho_{I,Y_{I}(1)})_+ \middle| T_I=0 \right] &\le \Pr \left[ p_{I,Y_{I}(1)}^{\mathrm{int}} \le \alpha \middle| T_I=0 \right]\le \alpha.
\end{align}
By definition of the IA-WCP prediction set in \eqref{eq:conf_set_trans_adj},
\begin{align}
	Y_{I}(1)\notin\Gamma_{I}^{\mathrm{IA-WCP}} \quad \Longleftrightarrow \quad p_{I,Y_{I}(1)}^{\mathrm{obs}} \le (\alpha-\rho_{I,Y_{I}(1)})_+.
\end{align}
Hence,
\begin{align}
	\Pr \left[ Y_{I}(1)\notin\Gamma_{I}^{\mathrm{IA-WCP}} \middle| T_I=0 \right] \le \alpha.
\end{align}
This proves Theorem \ref{thm:cov_trans_iawcp}.

\section{Coverage Tightness of the Correction Factor}
\label{sec:trans_equal_coverage}

We give an example in which the IA-WCP set in \eqref{eq:conf_set_trans_adj} is tight in the sense that it achieves the same marginal coverage as the ideal WCP set in \eqref{eq:ideal_WCP_set}. Fix a population size $N\ge 2$ and set the miscoverage level to $\alpha=\alpha_N=1/N$. For a constant propensity $\pi\in(0,1)$, covariates and treatments are generated as
\begin{align}
	\label{eq:equal_coverage_x_t}
	X_i&\overset{\mathrm{iid}}{\sim}\operatorname{Unif}(0,1), & T_i&\overset{\mathrm{iid}}{\sim}\operatorname{Bernoulli}(\pi), \qquad i\in\{1,\dots,N\}.
\end{align}
Define the permutation-equivariant neighborhood rule
\begin{align}
    \mathcal N_i(X_{1:N})=\{\ell\ne i:X_\ell<X_i\}.
\end{align}
We consider an exposure mechanism in which the exposure $E_i$ of unit $i$ is determined by the number of control units in this neighborhood, i.e.,
\begin{align}
	E_i &= \sum_{\ell\ne i}(1-T_\ell)\mathds 1\{X_\ell<X_i\}.
	\label{eq:equal_coverage_exposure}
\end{align}
The potential outcomes are defined as
\begin{align}
	Y_i(1) &= \begin{cases} X_i, & E_i=0,\\ -1, & E_i\ge 1, \end{cases} & Y_i(0)&=0.
	\label{eq:equal_coverage_outcomes}
\end{align}
The deterministic exposure and outcome kernels also satisfy the conditional-independence requirements in Assumptions \ref{ass:exposure_mechanism} and \ref{ass:no_hidden_interference_confounding}. Take $\mathcal Y=[-1,1]$ and use the nonconformity score
\begin{align}
	S(x,e,y)=y.
	\label{eq:equal_coverage_score}
\end{align}

We use the ideal WCP set $\tilde\Gamma_I^{\mathrm{WCP}}$ in \eqref{eq:ideal_WCP_set} and the IA-WCP set $\Gamma_I^{\mathrm{IA-WCP}}$ in \eqref{eq:conf_set_trans_adj}, both evaluated at miscoverage level $\alpha_N$.

\begin{proposition}
	\label{prop:trans_equal_coverage}
	Under the construction in \eqref{eq:equal_coverage_x_t}--\eqref{eq:equal_coverage_score}, conditionally on $T_I=0$, the adjusted and interventional prediction sets have the same marginal coverage,
	\begin{align}
		\Pr\left[Y_I(1)\in\tilde\Gamma_I^{\mathrm{WCP}}\mid T_I=0\right] = \Pr\left[Y_I(1)\in\Gamma_I^{\mathrm{IA-WCP}}\mid T_I=0\right] = 1-\frac{\pi^{N-1}}{N} \ge 1-\alpha_N .
		\label{eq:equal_coverage_result}
	\end{align}
\end{proposition}

\begin{proof}
	Let $I$ be drawn uniformly from $\{1,\dots,N\}$ and condition on $T_I=0$.  Let
	\begin{align}
		\mathcal A_I = \{T_j=1\text{ for every }j\ne I\}
	\end{align}
	be the event that the target is the only control.  Conditional on $T_I=0$,
	\begin{align}
		\Pr(\mathcal A_I\mid T_I=0)=\pi^{N-1}.
		\label{eq:equal_coverage_event_probability}
	\end{align}

	On $\mathcal A_I$, the target exposure is zero, so its true treated score is $V_I(Y_I(1))=X_I$.  If a treated calibration unit $j$ satisfies $X_j<X_I$, then its observational and interventional scores both equal $X_j<X_I$.  If $X_j>X_I$, its observational exposure is one because of the control target, whereas its interventional exposure is zero after setting $T_I=1$.  Therefore,
	\begin{align}
		V_j^{\mathrm{obs}}=-1<X_I<X_j=V_j^{\mathrm{int}}.
	\end{align}
	Consequently, on $\mathcal A_I$,
	\begin{align}
		\mathcal H_I=\mathcal H_I^{\mathrm{ex}}=\{j\ne I:X_j>X_I\}.
	\end{align}
	Writing $H=|\mathcal H_I|$ and using the equality of the weights gives
	\begin{align}
		p_{I,Y_I(1)}^{\mathrm{obs}} = \frac{1}{N}, \qquad \rho_{I,Y_I(1)} = \frac{H}{N}, \qquad p_{I,Y_I(1)}^{\mathrm{int}} = \frac{H+1}{N}.
		\label{eq:equal_coverage_rank_identity}
	\end{align}
	Thus $p_{I,Y_I(1)}^{\mathrm{obs}}+\rho_{I,Y_I(1)} =p_{I,Y_I(1)}^{\mathrm{int}}$ at the true outcome.  Since $p_{I,Y_I(1)}^{\mathrm{obs}}>0$, the adjusted acceptance condition is equivalent to
	\begin{align}
		p_{I,Y_I(1)}^{\mathrm{obs}}+\rho_{I,Y_I(1)}>\alpha_N .
	\end{align}
	Hence, the ideal and corrected observational procedures make the same coverage decision on $\mathcal A_I$.  Moreover, $H+1$ is the descending rank of $X_I$ among $N$ i.i.d. continuous observations, and is therefore uniform on $\{1,\ldots,N\}$.  At $\alpha_N =1/N$, both procedures miscover on $\mathcal A_I$ precisely when $H=0$, which has conditional probability $1/N$.

	On $\mathcal A_I^c$, there is at least one control in addition to $I$, so the number of treated calibration units is at most $N-2$, i.e., $|\mathcal I_1|\le N-2$.  Both the interventional and observational $p$-values are bounded below as
	\begin{align}
		p_{I,Y_I(1)}^{\mathrm{obs}}, p_{I,Y_I(1)}^{\mathrm{int}} \ge \frac{1}{1+|\mathcal I_1|} \ge \frac{1}{N-1} > \frac{1}{N} = \alpha_N .
	\end{align}
	The ideal and corrected observational procedures therefore both cover on $\mathcal A_I^c$.  Combining this with  \eqref{eq:equal_coverage_event_probability} proves \eqref{eq:equal_coverage_result}.
\end{proof}

The empirical coverage of the ideal WCP, WCP, and the proposed IA-WCP is illustrated in Figure \ref{fig:equal_coverage}. As established in Proposition \ref{prop:trans_equal_coverage}, the ideal and corrected procedures have identical empirical coverage, while the unadjusted observational procedure can substantially undercover for smaller values of $N$, although the discrepancy decreases as $N$ grows.
\begin{figure}[h!]
	\centering
	\includegraphics[width=0.5\textwidth]{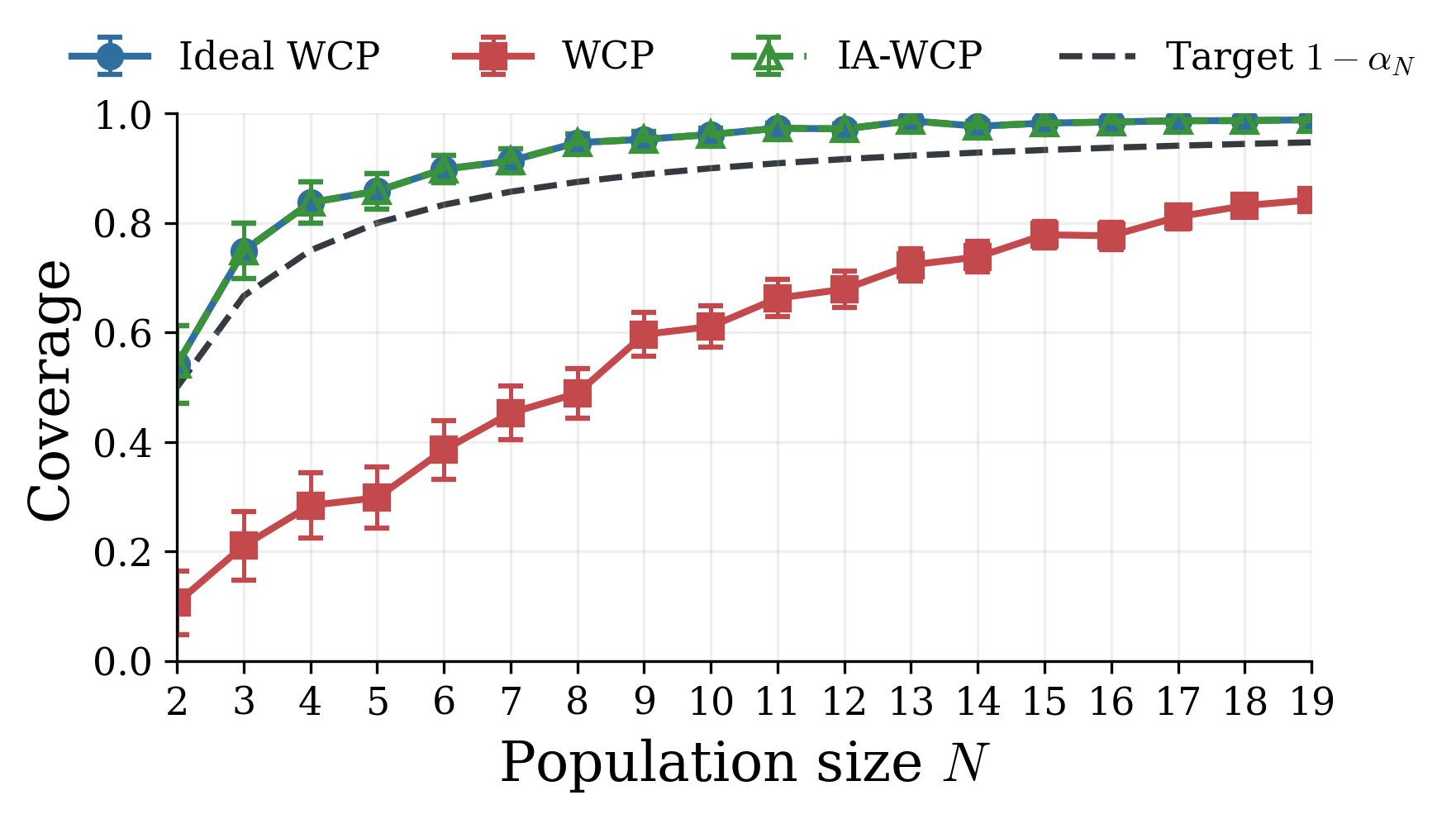}
	\caption{Empirical marginal coverage of the ideal interventional, corrected observational, and unadjusted observational procedures under the setting of Proposition \ref{prop:trans_equal_coverage}, with $\pi=0.9$ and $\alpha_N=1/N$. Error bars represent $95\%$ Monte Carlo confidence intervals. The ideal and corrected procedures have identical empirical coverage, in agreement with Proposition \ref{prop:trans_equal_coverage}. The unadjusted observational procedure can substantially undercover for smaller values of $N$, although the discrepancy decreases as $N$ grows.}
	\label{fig:equal_coverage}	
\end{figure}

\section{Transductive IA-WCP+}
\label{app:ia_wcp_plus_trans}

As mentioned in Section \ref{subsec:trans_ia_wcp}, IA-WCP+ leverages bounds on intervention-induced changes in the affected calibration scores to produce a sharper correction than the worst-case correction of IA-WCP. In this section, we formalize the stability condition under which the refined correction $\rho_{I,y}^{+}$ introduced in \eqref{eq:rho_delta_trans_plus} preserves marginal coverage. Throughout, we use the notation of Section \ref{sec:conf_tran_ite}, condition on $T_I=0$, and impose the following stability assumption on the affected calibration scores.

\begin{assumption}
	\label{ass:delta_score_stability_trans}
	For every affected calibration unit $j\in\mathcal H_I$, there exists a finite, nonnegative, observable margin $\Delta_j^I$ that is a measurable function of $(X_{1:N},T_{1:N},I,E_j,Y_j(1))$, but does not depend on the random exposures or outcomes of other units, such that for every $t\in\mathbb R$,
	\begin{align}
\Pr\left[V_j^{\mathrm{int}}>t\mid X_{1:N},T_{1:N},I\right]\le\Pr\left[V_j^{\mathrm{obs}}+\Delta_j^I>t\mid X_{1:N},T_{1:N},I\right].
		\label{eq:delta_score_tail_condition}
	\end{align}
\end{assumption}

Assumption \ref{ass:delta_score_stability_trans} states that, conditionally on the covariates, treatment assignments, and target index, the interventional score of each affected unit is stochastically dominated by its observational score shifted upward by $\Delta_j^I$. As shown in Section \ref{subsec:suff_conf_stability_trans}, sufficient conditions for this assumption include Lipschitz continuity of the nonconformity score and almost-sure bounds on the intervention-induced changes in exposures and potential outcomes.

Using margins satisfying \eqref{eq:delta_score_tail_condition} in the correction factor \eqref{eq:rho_delta_trans_plus}, IA-WCP+ defines the prediction set as
\begin{align}
\Gamma_I^{\mathrm{IA-WCP}+}
=\left\{y\in\mathcal Y:p_{I,y}^{\mathrm{obs}}>(\alpha-\rho_{I,y}^{+})_+\right\},
\qquad \alpha\in(0,1).
\label{eq:plus_trans_set}
\end{align}
Since $0\le\rho_{I,y}^{+}\le\rho_{I,y}$ for every $y$, comparing the acceptance thresholds in \eqref{eq:plus_trans_set} and \eqref{eq:conf_set_trans_adj} gives
\begin{align}
\Gamma_I^{\mathrm{IA-WCP}+}\subseteq\Gamma_I^{\mathrm{IA-WCP}}.
\end{align}
Thus, IA-WCP+ returns a prediction set no larger than IA-WCP for the same observational sample. Under the assumptions of the following lemma, $p_{I,y}^{\mathrm{obs}}+\rho_{I,y}^{+}$ bounds the interventional $p$-value from above under a suitable coupling.

\begin{lemma}
	\label{lem:trans_p_value_comparison_delta}
	Under Assumptions \ref{ass:exposure_mechanism}, \ref{ass:no_hidden_interference_confounding}, and \ref{ass:delta_score_stability_trans}, conditionally on $T_I=0$, the observational and interventional conformal $p$-values admit a coupling such that, simultaneously for every $y\in\mathcal Y$,
	\begin{align}
		p_{I,y}^{\mathrm{int}} \le p_{I,y}^{\mathrm{obs}}+\rho_{I,y}^{+} \qquad\text{a.s.}
	\end{align}
\end{lemma}

\begin{proof}
	Work conditionally on $(X_{1:N},T_{1:N},I)$ with $T_I=0$, and write $\mathcal U_I=\mathcal I_1\setminus\mathcal H_I$. As in Lemma \ref{lem:trans_p_value_comparison:merged}, the observational and interventional scores can be coupled so that
	\begin{align}
		V_I^{\mathrm{int}}(y)=V_I(y),\quad\text{ and } \quad V_j^{\mathrm{obs}} = V_j^{\mathrm{int}} \qquad\text{a.s. for every }j\in\mathcal U_I.
	\end{align}

	For every $j\in\mathcal H_I$, Assumption \ref{ass:delta_score_stability_trans} and the coupling characterization of stochastic order \citep{strassen1965existence} imply the existence of a coupling, preserving the conditional joint law of $(V_j^{\mathrm{obs}},\Delta_j^I)$, such that
	\begin{align}
		V_j^{\mathrm{int}} \le V_j^{\mathrm{obs}}+\Delta_j^I \qquad\text{a.s.}
		\label{eq:delta_score_coupling}
	\end{align}
	Since, conditional on $(X_{1:N},T_{1:N},I)$, each $\Delta_j^I$ may depend only on $(E_j,Y_j(1))$ and not on the exposures or outcomes of other units, the pair $(V_j^{\mathrm{obs}},\Delta_j^I)$ is a function only of the random variables $(E_j,Y_j(1))$. From Assumptions \ref{ass:exposure_mechanism} and \ref{ass:no_hidden_interference_confounding} it then follows that the pairs $(V_j^{\mathrm{obs}},\Delta_j^I)$ are conditionally independent across units. The unit couplings can therefore be combined into a product coupling under which \eqref{eq:delta_score_coupling} holds simultaneously for every $j\in\mathcal H_I$.

	The interventional scores are conditionally independent for the same reason. The product coupling therefore preserves both score-vector marginals. We also retain the identical target coupling $(\tilde E_I^I,\tilde Y_I^I(1))=(E_I,Y_I(1))$ from Lemma \ref{lem:trans_p_value_comparison:merged}, preserving the joint laws with the target potential outcome.

	For every $j\in\mathcal H_I$, \eqref{eq:delta_score_coupling} implies
	\begin{align}
		\mathds 1\{V_j^{\mathrm{int}}\ge V_I(y)\} - \mathds 1\{V_j^{\mathrm{obs}}\ge V_I(y)\}\le \mathds 1 \left\{ V_I(y)-\Delta_j^I \le V_j^{\mathrm{obs}} < V_I(y) \right\}.
		\label{eq:delta_indicator_bound}
	\end{align}
	In fact, the left-hand side can be positive only if
	\begin{align}
		V_j^{\mathrm{obs}} < V_I(y) \le V_j^{\mathrm{int}} \le V_j^{\mathrm{obs}}+\Delta_j^I.
	\end{align}

	Since the unaffected scores coincide under the coupling,
	\begin{align}
		p_{I,y}^{\mathrm{int}}-p_{I,y}^{\mathrm{obs}} &= \frac{ \sum_{j\in\mathcal H_I} W_j \left[ \mathds 1\{V_j^{\mathrm{int}}\ge V_I(y)\} - \mathds 1\{V_j^{\mathrm{obs}}\ge V_I(y)\} \right] }{ W_I+\sum_{j\in\mathcal I_1}W_j }\le \rho_{I,y}^{+}.
	\end{align}
	Therefore,
	\begin{align}
		p_{I,y}^{\mathrm{int}} \le p_{I,y}^{\mathrm{obs}}+\rho_{I,y}^{+} \qquad\text{a.s.}
	\end{align}
	Since the coupling does not depend on $y$, the inequality holds simultaneously for every $y\in\mathcal Y$.
\end{proof}

It follows that IA-WCP+ preserves marginal coverage, as formalized in the following theorem.

\begin{theorem}
\label{thm:plus_trans_properties_reorganized}
Under Assumptions \ref{ass:prop_unit}, \ref{ass:eq_exp_map}, \ref{ass:exposure_mechanism}, \ref{ass:no_hidden_interference_confounding}, \ref{ass:consistency}, and \ref{ass:delta_score_stability_trans}, the prediction set in \eqref{eq:plus_trans_set} satisfies
\begin{align}
\Pr\left[Y_I(1)\notin\Gamma_I^{\mathrm{IA-WCP}+}\mid T_I=0\right]\le\alpha.
\label{eq:plus_trans_coverage}
\end{align}
\end{theorem}
\begin{proof}
The interventional superuniformity established in \eqref{eq:interventional_superuniform:merged} gives
\begin{align}
	\Pr\left[ p_{I,Y_{I}(1)}^{\mathrm{int}}\le \alpha\mid T_I=0 \right] \le \alpha.
\end{align}
By Lemma \ref{lem:trans_p_value_comparison_delta},
\begin{align}
	p_{I,y}^{\mathrm{int}} \le p_{I,y}^{\mathrm{obs}}+\rho^{+}_{I,y}.
\end{align}
Since overlap ensures $p_{I,y}^{\mathrm{obs}}>0$, rejection is impossible when $\rho_{I,y}^{+}\ge\alpha$. Otherwise, rejection implies $p_{I,y}^{\mathrm{obs}}+\rho_{I,y}^{+}\le\alpha$. Thus, under this coupling,
\begin{align}
	\left\{ p_{I,y}^{\mathrm{obs}} \le \left(\alpha-\rho^{+}_{I,y}\right)_+ \right\} \subseteq \left\{ p_{I,y}^{\mathrm{int}}\le \alpha \right\}.
\end{align}
The coupling preserves the joint law with $Y_I(1)$ and holds simultaneously for every $y$. Evaluating at $y=Y_I(1)$ therefore proves the claim.
\end{proof}

\subsection{Sufficient conditions for score stability}
\label{subsec:suff_conf_stability_trans}
Let $d_{\mathcal E}:\mathcal E\times\mathcal E\to[0,\infty)$ be a metric on the exposure space $\mathcal E$. Conditionally on $(X_{1:N},T_{1:N},I)$, suppose that the observational and interventional exposures and outcomes of every affected unit $j\in\mathcal H_I$ admit a coupling with the correct conditional marginals satisfying
\begin{align}
	d_{\mathcal E}(E_j,\tilde E_j^{I}) &\le \Delta_{E,j}^{I},\\
	\left|Y_j(1)-\tilde Y_j^{I}(1)\right| &\le \Delta_{Y,j}^{I}
\end{align}
almost surely, where the two bounds are finite, nonnegative, and observable, with the same local dependence allowed for $\Delta_j^I$ in Assumption \ref{ass:delta_score_stability_trans}. Suppose also that the nonconformity score satisfies, for finite constants $L_E,L_Y\ge0$,
\begin{align}
	\left|S(x,e,y)-S(x,e',y')\right| \le L_Ed_{\mathcal E}(e,e')+L_Y|y-y'|.
\end{align}
Then
\begin{align}
	\left|V_j^{\mathrm{int}}-V_j^{\mathrm{obs}}\right| \le L_E\Delta_{E,j}^{I}+L_Y\Delta_{Y,j}^{I}
\end{align}
almost surely. Hence, Assumption \ref{ass:delta_score_stability_trans} holds with
\begin{align}
	\Delta_j^{I} = L_E\Delta_{E,j}^{I}+L_Y\Delta_{Y,j}^{I}.
\end{align}
If the score does not depend on the exposure, the same conclusion holds with
\begin{align}
	\Delta_j^{I} = L_Y\Delta_{Y,j}^{I}.
\end{align}

\section{Inductive Interference-Adjusted Weighted Conformal Prediction}
\label{app:inductive_ia_wcp}
In this section, we consider the inductive setting described in Section \ref{sec:inductive_ite_estimation}. Given data $\mathcal D$ from a network of $N-1$ units, we construct a prediction set for the ITE of a new unit with covariates $X_N\sim P_X$ after it is embedded in the network. We outline how the reasoning developed for the transductive setting in Section \ref{sec:conf_tran_ite} extends to this case. In what follows, we construct a prediction set for the treated potential outcome $Y_N(1)$. 

As in the transductive setting, a direct application of WCP to the treated calibration units $\mathcal I_1=\{j\in\{1,\ldots,N-1\}:T_j=1\}$ does not generally guarantee coverage for $Y_N(1)$. In fact, inverse-propensity weighting with weights
\begin{align}
    W_j = \frac{1}{\pi(X_j)},\quad\text{ for } j\in\mathcal I_1\cup\{N\},
\end{align}
accounts for the difference between the covariate distributions of the treated calibration units and the new unit. However, the calibration exposures and outcomes are generated in the original network of $N-1$ units, whereas the test exposure $E_N$ and outcome $Y_N(1)$ are generated in the augmented network of $N$ units. Consequently, the calibration scores $V_j^{\mathrm{obs}}=S(X_j,E_j,Y_j(1))$, for $j\in\mathcal I_1$, and the test score $V_N(Y_N(1))$ need not be weighted exchangeable. Thus, the observational conformal $p$-value
\begin{align}
    p_{N,y}^{\mathrm{obs}}
    =\frac{W_N+\sum_{j\in\mathcal I_1}W_j\mathds 1\{V_j^{\mathrm{obs}}\ge V_N(y)\}}
    {W_N+\sum_{j\in\mathcal I_1}W_j}
\end{align}
need not be valid.

To correct this procedure, we follow the same interventional reasoning as in the transductive setting. Consider adding unit $N$ with treatment $T_N=1$ and regenerating the calibration exposures and potential outcomes according to the model in the augmented network. Under this intervention, calibration and test units become weighted exchangeable with weights $W_j$. Replacing the observational calibration scores in $p_{N,y}^{\mathrm{obs}}$ with their interventional counterparts yields an ideal WCP procedure with the desired marginal coverage guarantee. 

To relate the ideal but unobservable procedure to the observational one, IA-WCP defines the treated calibration units whose neighborhoods change when unit $N$ is added to the network
\begin{align}
    \mathcal H_N
    =\left\{j\in\mathcal I_1:
    \mathcal N_j(X_{1:N-1})\neq\mathcal N_j(X_{1:N})\right\},
\end{align}
and evaluates the correction factor
\begin{align}
    \label{eq:correction_factor_inductive}
    \rho_{N,y}
    =\frac{\sum_{j\in\mathcal H_N}W_j\mathds 1\{V_j^{\mathrm{obs}}<V_N(y)\}}
    {W_N+\sum_{j\in\mathcal I_1}W_j}.
\end{align}
The factor \eqref{eq:correction_factor_inductive} accounts for the discrepancy between the observational and interventional conformal $p$-values and it is used to define the IA-WCP prediction set
\begin{align}
    \label{eq:iawcp_inductive_set}
    \Gamma_N^{\mathrm{IA-WCP}}
    =\left\{y\in\mathcal Y:
    p_{N,y}^{\mathrm{obs}}>(\alpha-\rho_{N,y})_+\right\}.
\end{align}
A suitable coupling of the observational and interventional scores ensures that the IA-WCP set contains the ideal WCP set, yielding the following coverage guarantee.

\begin{theorem}[Inductive IA-WCP coverage]
    \label{thm:iawcp_inductive_coverage}
    Under Assumptions \ref{ass:prop_unit}--\ref{ass:consistency}, the IA-WCP set $\Gamma_N^{\mathrm{IA-WCP}}$ in \eqref{eq:iawcp_inductive_set} satisfies the inequality
    \begin{align}
        \Pr\left[Y_N(1)\in\Gamma_N^{\mathrm{IA-WCP}}\right]\ge 1-\alpha.
    \end{align}
\end{theorem}
\begin{proof}
    See Appendix \ref{proof:inductive_appendix}.
\end{proof}

For $Y_N(0)$, the corresponding prediction set is obtained using the control units and weights $1/(1-\pi(X_j))$. To distinguish the two constructions, we now make the target treatment explicit in the notation. Let $\Gamma_N^{t,\mathrm{IA-WCP}}$ denote the prediction set for $Y_N(t)$ constructed at miscoverage level $\alpha_t$. Choosing $\alpha_1+\alpha_0=\alpha$, the ITE prediction set is defined as
\begin{align}
    \Gamma_N^{\mathrm{ITE}}
    =\left\{y_1-y_0:
    y_1\in\Gamma_N^{1,\mathrm{IA-WCP}},
    y_0\in\Gamma_N^{0,\mathrm{IA-WCP}}\right\},
\end{align}
and it satisfies \eqref{eq:coverage_inductive_id} by the union bound. 
\subsection{Inductive IA-WCP+}
As in the transductive setting, the correction can be sharpened when the intervention-induced increase in each affected calibration score is stochastically bounded.

\begin{assumption}
	\label{ass:delta_score_stability_ind}
	For every affected calibration unit $j\in\mathcal H_{N}$, there exists a finite, nonnegative, observable margin $\Delta_j^{N}$ that is a measurable function of $(X_{1:N},T_{1:N},E_j,Y_j(1))$, but does not depend on the random exposures or outcomes of other units, such that for every $t\in\mathbb R$,
	\begin{align}
		\Pr\left[ V_j^{\mathrm{int}}>t \mid X_{1:N},T_{1:N} \right] \le \Pr\left[ V_j^{\mathrm{obs}}+\Delta_j^{N}>t \mid X_{1:N},T_{1:N} \right].
		\label{eq:delta_score_tail_condition_ind}
	\end{align}
\end{assumption}

Assumption \ref{ass:delta_score_stability_ind} states that, conditionally on the augmented covariate and treatment vectors, the interventional score of each affected calibration unit is stochastically dominated by its observational score shifted upward by $\Delta_j^{N}$. Using the stability margins in \eqref{eq:delta_score_tail_condition_ind}, IA-WCP+ defines the refined correction factor
\begin{align}
	\rho_{N,y}^{+} = \frac{ \sum_{j\in\mathcal H_{N}} W_j \mathds 1 \left\{ V_{N}(y)-\Delta_j^{N} \le V_j^{\mathrm{obs}} < V_{N}(y) \right\} }{ W_{N}+\sum_{j\in\mathcal I_1}W_j }.
	\label{eq:rho_delta_ind}
\end{align}
Under Assumption \ref{ass:delta_score_stability_ind}, the observational and interventional score vectors admit a coupling with the correct conditional marginals such that, simultaneously for every $y\in\mathcal Y$,
\begin{align}
	p_{N,y}^{\mathrm{int}} \le p_{N,y}^{\mathrm{obs}}+\rho_{N,y}^{+} \qquad\text{a.s.}
	\label{eq:p_value_delta_comparison_ind}
\end{align}
The correction $\rho_{N,y}^{+}$ counts only the conformal weight of affected calibration units whose observational scores lie within a margin $\Delta_j^{N}$ below the test score.

Under this coupling, every candidate accepted by the ideal interventional procedure is also accepted by the adjusted observational procedure. This motivates the IA-WCP+ prediction set
\begin{align}
	\Gamma_N^{\mathrm{IA-WCP+}} = \left\{ y\in\mathcal Y: p_{N,y}^{\mathrm{obs}} > (\alpha-\rho_{N,y}^{+})_+ \right\}.
	\label{eq:conf_set_ind_delta_adj}
\end{align}
Because $\rho_{N,y}^{+}\le\rho_{N,y}$, the IA-WCP+ prediction set $\Gamma_N^{\mathrm{IA-WCP+}}$ is contained in the IA-WCP prediction set $\Gamma_N^{1,\mathrm{IA-WCP}}$ in \eqref{eq:iawcp_inductive_set}.

\begin{theorem}
	\label{thm:cov_loss_ind_ite_delta}
	Under Assumptions \ref{ass:prop_unit}--\ref{ass:consistency} and \ref{ass:delta_score_stability_ind}, the IA-WCP+ prediction set in \eqref{eq:conf_set_ind_delta_adj} satisfies
	\begin{align}
		\Pr\left[ Y_{N}(1)\notin\Gamma_N^{\mathrm{IA-WCP+}} \right] \le \alpha.
	\end{align}
\end{theorem}

\begin{proof}
See Appendix \ref{proof:thm:cov_loss_ind_ite_delta}.
\end{proof}

\section{Proofs of Inductive IA-WCP and Inductive IA-WCP+ Validity}
\label{proof:inductive_appendix}

In this section, we provide the proofs of Theorem \ref{thm:iawcp_inductive_coverage} and Theorem \ref{thm:cov_loss_ind_ite_delta}. Throughout, the observational network contains $N-1$ units and the new unit is indexed by $N$, as in the main text. Following the same template as the transductive case, we first establish the validity of ideal WCP in the augmented network, then compare its scores with the observational scores through a coupling and establish the validity of IA-WCP. In Appendix \ref{proof:thm:cov_loss_ind_ite_delta} we show that the same reasoning applies to IA-WCP+ under the score stability condition stated in Assumption \ref{ass:delta_score_stability_ind}.

\subsection{Validity of the ideal WCP procedure}
\label{app:inductive_interventional_sample}

We consider the case in which the test unit is embedded into the observed network with treatment $T_{N}=1$ and we wish to estimate its potential outcome $Y_{N}(1)$. Define the augmented covariate and treatment vectors
\begin{align}
	X_{1:N} = (X_1,\ldots,X_{N-1},X_{N}), \qquad T_{1:N} = (T_1,\ldots,T_{N-1},1),
\end{align}
where $X_N\sim P_X$ is independent of the observed covariates $X_{1:N-1}$ and treatments $T_{1:N-1}$.
For each treated calibration unit $j\in\mathcal I_1$, let $\tilde E_j$ denote the exposure of unit $j$ in the augmented network, and let $\tilde Y_j(1)$ denote the corresponding treated potential outcome. Thus,
\begin{align}
	\tilde E_j \mid (X_{1:N},T_{1:N}) \sim P_{E\mid X,  \mathcal{L}}\left(\cdot\mid X_j,  \mathcal{L}_j(X_{1:N},T_{1:N})\right),
\end{align}
and
\begin{align}
	\tilde Y_j(1)\mid X_j,\tilde E_j \sim P_{Y(1)\mid X,E}(\cdot\mid X_j,\tilde E_j).
\end{align}
For the test unit, let
\begin{align}
	E_{N} \mid (X_{1:N},T_{1:N}) \sim P_{E\mid X,  \mathcal{L}}\left(\cdot\mid X_{N},  \mathcal{L}_{N}(X_{1:N},T_{1:N})\right),
\end{align}
and
\begin{align}
	Y_{N}(1)\mid X_{N},E_{N} \sim P_{Y(1)\mid X,E}(\cdot\mid X_{N},E_{N}).
\end{align}
We define the interventional nonconformity scores $V_j^{\mathrm{int}}=S(X_j,\tilde E_j,\tilde Y_j(1))$ and the triples
\begin{align}
    \tilde Z_j = (X_j,\tilde E_j,\tilde Y_j(1)), \qquad j\in\mathcal I_1,
\end{align}
and
\begin{align}
	Z_{N} = (X_{N},E_{N},Y_{N}(1)).
\end{align}

We first establish the weighted exchangeability of the augmented interventional sample.

\begin{lemma}[Inductive weighted exchangeability]
	\label{lem:ind_weighted_exchangeability}
	Under Assumptions \ref{ass:prop_unit}, \ref{ass:exposure_mechanism}, \ref{ass:no_hidden_interference_confounding}, and \ref{ass:eq_exp_map}, the joint density of
	\begin{align}
		\bigl((\tilde Z_j)_{j\in\mathcal I_1},Z_N\bigr),
	\end{align}
	conditional on the realized treated set $\mathcal I_1$, admits the decomposition
	\begin{align}
		p\left( (\tilde Z_j)_{j\in\mathcal I_1}, Z_{N} \mid \mathcal I_1 \right) = cw(X_{N}) g\left( (\tilde Z_j)_{j\in\mathcal I_1}, Z_{N} \right),
	\end{align}
	where $c>0$ is a normalizing constant,
	\begin{align}
		w(x)=\frac{1}{\pi(x)},
	\end{align}
	and $g(\cdot)$ is invariant under permutations of its arguments.
\end{lemma}

\begin{proof}
Conditionally on the event defining the treated calibration set $\mathcal I_1$, the covariates of treated calibration units have density proportional to $\pi(x)p_X(x)$. In contrast, the test-unit covariate has density $p_X(x)$. Since
\begin{align}
	p_X(X_{N}) = w(X_{N})\pi(X_{N})p_X(X_{N}), \qquad w(x)=\frac{1}{\pi(x)},
\end{align}
the covariate density of the augmented sample can be written, up to a normalizing constant, as
\begin{align}
	w(X_{N}) \prod_{j\in\mathcal I_1}\pi(X_j)p_X(X_j) \pi(X_{N})p_X(X_{N}).
\end{align}
After extracting the factor $w(X_{N})$, the remaining covariate density is symmetric in the augmented covariates $\{X_j:j\in\mathcal I_1\}\cup\{X_{N}\}$.

Under the augmented treatment vector $T_{1:N}$, all units in $\mathcal I_1\cup\{N\}$ have treatment equal to one. By Assumption \ref{ass:exposure_mechanism}, the exposure variables are conditionally independent given the augmented covariates and treatments. By Assumption \ref{ass:no_hidden_interference_confounding}, the corresponding treated potential outcomes are conditionally independent given their covariates and exposures, with common conditional distribution $P_{Y(1)\mid X,E}$.

It remains to check the invariance of the remaining factor. Since all units in the augmented set $\mathcal I_1\cup\{N\}$ have treatment equal to one, permuting these units only relabels their positions.  By the permutation equivariance of the neighborhood rule in Assumption \ref{ass:eq_exp_map} and the ordering invariance of the common exposure kernel, the joint exposure distribution is unchanged under a simultaneous permutation of the triples
\begin{align}
	(X_j,\tilde E_j,\tilde Y_j(1)), \qquad j\in\mathcal I_1,
\end{align}
and the test triple $(X_{N},E_{N},Y_{N}(1))$. Covariates and variables of the control units are integrated out; their contribution is unchanged by such a relabeling. Hence the remaining factor $g(\cdot)$ is invariant under permutations of the augmented sample, proving the claim.
\end{proof}

By Lemma \ref{lem:ind_weighted_exchangeability}, the augmented interventional sample is weighted exchangeable with weights $w(x)=1/\pi(x)$. Therefore, the standard weighted conformal validity argument \citep{lei2021conformal} implies that the ideal interventional conformal $p$-value
\begin{align}
    \label{eq:inductive_ideal_p_appendix}
	p_{N,y}^{\mathrm{int}} = \frac{ W_{N} + \sum_{j\in\mathcal I_1} W_j \mathds 1\{V_j^{\mathrm{int}}\ge V_{N}(y)\} }{ W_{N} + \sum_{j\in\mathcal I_1}W_j }
\end{align}
is valid and the ideal prediction set
\begin{align}
    \tilde\Gamma_N^{\mathrm{WCP}}=\{y\in\mathcal Y:p_{N,y}^{\mathrm{int}}>\alpha\}
\end{align}
satisfies 
\begin{align}
   \Pr[Y_N(1)\in\tilde\Gamma_N^{\mathrm{WCP}}]\ge1-\alpha.
\end{align}

\subsection{Coupling comparison}
\label{proof:lemma_inductive}

Recall the definition of the units affected by the addition of unit $N$ to the network,
\begin{align}
    \mathcal H_N=\{j\in\mathcal I_1:\mathcal N_j(X_{1:N-1})\neq\mathcal N_j(X_{1:N})\},
\end{align}
and write $\mathcal U_N=\mathcal I_1\setminus\mathcal H_N$. For the proof, we consider the tighter exposure-law affected set
\begin{align}
    \mathcal H_{N}^{\mathrm{ex}} = \Bigl\{j\in\mathcal I_1: &P_{E\mid X,  \mathcal{L}} \left(\cdot\mid X_j,  \mathcal{L}_j(X_{1:N-1},T_{1:N-1})\right) \neq P_{E\mid X,  \mathcal{L}} \left(\cdot\mid X_j,  \mathcal{L}_j(X_{1:N},T_{1:N})\right) \Bigr\},
    \label{eq:exact_affected_ind}
\end{align}
and let $\mathcal U_{N}^{\mathrm{ex}}=\mathcal I_1\setminus\mathcal H_{N}^{\mathrm{ex}}$. Since equal neighborhoods give equal local configurations before and after augmentation, $ \mathcal H_{N}^{\mathrm{ex}}\subseteq \mathcal H_{N}.$
The tighter set $\mathcal H_{N}^{\mathrm{ex}}$ is used to construct the coupling, after which the index set of the nonnegative correction sum is enlarged from $\mathcal H_{N}^{\mathrm{ex}}$ to $ \mathcal H_{N}$.

\begin{lemma}
	\label{lem:ind_p_value_comparison}
	Under Assumptions \ref{ass:exposure_mechanism} and \ref{ass:no_hidden_interference_confounding}, define the observational conformal $p$-value by 
	\begin{align}
	\label{eq:conf_p_value_ind}
	p_{N,y}^{\mathrm{obs}} = \frac{W_{N}+\sum_{j\in\mathcal I_1}W_j\mathds 1\{V_j^{\mathrm{obs}}\ge V_{N}(y)\}}{W_{N}+\sum_{j\in\mathcal I_1}W_j}.
	\end{align}
	 Then this $p$-value and the interventional conformal $p$-value in \eqref{eq:inductive_ideal_p_appendix} admit a coupling such that, simultaneously for every $y\in\mathcal Y$,
	\begin{align}
		p_{N,y}^{\mathrm{int}} \le p_{N,y}^{\mathrm{obs}}+\rho_{N,y} \qquad\text{a.s.},
	\end{align}
	where $\rho_{N,y}$ is defined in \eqref{eq:correction_factor_inductive}.
\end{lemma}

\begin{proof}
	We construct the coupling conditionally on $(X_{1:N},T_{1:N})$. For every $j\in\mathcal U_{N}^{\mathrm{ex}}$, the definition of $\mathcal U_{N}^{\mathrm{ex}}$ gives
	\begin{align}
		P_{E\mid X,  \mathcal{L}} \left( \cdot \mid {X_j,  \mathcal{L}_j(X_{1:N-1},T_{1:N-1})} \right) = P_{E\mid X,  \mathcal{L}} \left( \cdot \mid {X_j,  \mathcal{L}_j(X_{1:N},T_{1:N})} \right).
	\end{align}
	Thus, for each $j\in\mathcal U_{N}^{\mathrm{ex}}$, the observational and interventional exposures may be coupled so that
	\begin{align}
		E_j = \tilde E_j \qquad\text{a.s.}
	\end{align}
	Conditional on these coupled exposures, the corresponding treated potential outcomes have the same conditional distribution $P_{Y(1)\mid X,E}$ and may therefore be coupled so that
	\begin{align}
		Y_j(1) = \tilde Y_j(1) \qquad\text{a.s.}
	\end{align}
    For every $j\in\mathcal H_{N}^{\mathrm{ex}}$, choose any coupling with the correct observational and interventional marginals.

	By Assumptions \ref{ass:exposure_mechanism} and \ref{ass:no_hidden_interference_confounding}, the unit-specific exposure and outcome variables are conditionally independent across units. These couplings can therefore be combined into a product coupling that preserves the observational and interventional score-vector marginals. Under this coupling,
	\begin{align}
		V_j^{\mathrm{obs}} = V_j^{\mathrm{int}} \qquad\text{a.s. for every }{ j\in\mathcal U_{N}^{\mathrm{ex}}}.
	\end{align}
	The test exposure and outcome have the same joint law with the augmented covariates and treatments in both constructions, and are coupled identically. The two $p$-values therefore use the same target score $V_{N}(y)$ and the same weights, since the weights depend only on the covariates. Consequently,
	\begin{align}
		p_{N,y}^{\mathrm{int}} - p_{N,y}^{\mathrm{obs}} &= \frac{ {\sum_{j\in\mathcal H_{N}^{\mathrm{ex}}}} W_j \left[ \mathds 1\{V_j^{\mathrm{int}}\ge V_{N}(y)\} - \mathds 1\{V_j^{\mathrm{obs}}\ge V_{N}(y)\} \right] }{ W_{N}+\sum_{j\in\mathcal I_1}W_j }.
	\end{align}
	For every $j\in\mathcal H_{N}^{\mathrm{ex}}$,
	\begin{align}
		\mathds 1\{V_j^{\mathrm{int}}\ge V_{N}(y)\} - \mathds 1\{V_j^{\mathrm{obs}}\ge V_{N}(y)\} \le \mathds 1\{V_j^{\mathrm{obs}}<V_{N}(y)\}.
	\end{align}
	Therefore,
	\begin{align}
		p_{N,y}^{\mathrm{int}} - p_{N,y}^{\mathrm{obs}} \le \frac{ \sum_{j\in\mathcal H_{N}^{\mathrm{ex}}} W_j \mathds 1\{V_j^{\mathrm{obs}}<V_{N}(y)\} }{ W_{N}+\sum_{j\in\mathcal I_1}W_j }&\le \frac{ \sum_{j\in\mathcal H_{N}} W_j \mathds 1\{V_j^{\mathrm{obs}}<V_{N}(y)\} }{ W_{N}+\sum_{j\in\mathcal I_1}W_j}=\rho_{N,y}.
	\end{align}
	Hence,
	\begin{align}
		p_{N,y}^{\mathrm{int}} \le p_{N,y}^{\mathrm{obs}}+\rho_{N,y} \qquad\text{a.s.}
	\end{align}
	Since the coupling does not depend on $y$, this inequality holds simultaneously for every $y\in\mathcal Y$.
\end{proof} 

\subsection{Proof of Theorem \ref{thm:iawcp_inductive_coverage}}
\label{proof:cov_loss_ind_ite}

By Lemma \ref{lem:ind_p_value_comparison}, there is a coupling preserving the observational and interventional marginals such that $p_{N,y}^{\mathrm{int}}\le p_{N,y}^{\mathrm{obs}}+\rho_{N,y}$ simultaneously for every $y\in\mathcal Y$. If $\rho_{N,y}<\alpha$, then $p_{N,y}^{\mathrm{int}}>\alpha$ implies $p_{N,y}^{\mathrm{obs}}>\alpha-\rho_{N,y}$. If $\rho_{N,y}\ge\alpha$, the IA-WCP acceptance threshold is zero, and $p_{N,y}^{\mathrm{obs}}>0$ because $W_N>0$. Hence,
\begin{align}
    \tilde\Gamma_N^{\mathrm{WCP}}\subseteq\Gamma_N^{\mathrm{IA-WCP}}\qquad\text{a.s.}
\end{align}
Evaluating this inclusion at the shared target outcome and using the validity of the interventional $p$-value gives
\begin{align}
    \Pr[Y_N(1)\notin\Gamma_N^{\mathrm{IA-WCP}}]
    \le\Pr[p_{N,Y_N(1)}^{\mathrm{int}}\le\alpha]\le\alpha,
\end{align}
which proves Theorem \ref{thm:iawcp_inductive_coverage}.

For the control potential outcome, fix the target treatment at zero, use the control calibration units and weights $1/(1-\pi(X_j))$, and repeat the argument. Write $\Gamma_N^{t,\mathrm{IA-WCP}}$ for the set targeting $Y_N(t)$ at miscoverage level $\alpha_t$. Since the neighborhood excludes the unit itself, the target exposure law is unchanged by its own treatment; both potential-outcome guarantees therefore apply to the common target exposure in the ITE definition. If both potential outcomes belong to their respective prediction sets, their difference belongs to $\Gamma_N^{\mathrm{ITE}}$. Thus,
\begin{align}
    \Pr[Y_N(1)-Y_N(0)\notin\Gamma_N^{\mathrm{ITE}}]
    &\le\Pr[Y_N(1)\notin\Gamma_N^{1,\mathrm{IA-WCP}}]
    +\Pr[Y_N(0)\notin\Gamma_N^{0,\mathrm{IA-WCP}}]\le\alpha_1+\alpha_0=\alpha.
\end{align}

\subsection{Proof of Theorem \ref{thm:cov_loss_ind_ite_delta}}

\label{proof:thm:cov_loss_ind_ite_delta}
For every $j\in\mathcal U_{N}$, the coupling in the proof of Lemma \ref{lem:ind_p_value_comparison} gives
\begin{align}
	V_j^{\mathrm{int}} = V_j^{\mathrm{obs}} \qquad\text{a.s.}
\end{align}
For every $j\in\mathcal H_{N}$, Assumption \ref{ass:delta_score_stability_ind} establishes the required conditional stochastic ordering. By the coupling characterization of stochastic order \citep{strassen1965existence}, applied conditionally on $(X_{1:N},T_{1:N})$, there exists a coupling of $V_j^{\mathrm{int}}$ and $(V_j^{\mathrm{obs}},\Delta_j^{N})$ with the correct conditional marginals such that
\begin{align}
	V_j^{\mathrm{int}} \le V_j^{\mathrm{obs}}+\Delta_j^{N} \qquad\text{a.s.}
	\label{eq:delta_score_coupling_ind}
\end{align}

Conditionally on $(X_{1:N},T_{1:N})$, each pair $(V_j^{\mathrm{obs}},\Delta_j^{N})$ depends only on $(E_j,Y_j(1))$. Assumptions \ref{ass:exposure_mechanism} and \ref{ass:no_hidden_interference_confounding} therefore imply that these pairs are conditionally independent across units. The interventional scores are conditionally independent for the same reason. Hence, the coordinatewise couplings can be combined into a product coupling under which \eqref{eq:delta_score_coupling_ind} holds simultaneously for every $j\in\mathcal H_{N}$, while preserving the observational and interventional score-vector marginals.

For every $j\in\mathcal H_{N}$, \eqref{eq:delta_score_coupling_ind} implies
\begin{align}
	\mathds 1\{V_j^{\mathrm{int}}\ge V_{N}(y)\} - \mathds 1\{V_j^{\mathrm{obs}}\ge V_{N}(y)\}\le \mathds 1 \left\{ V_{N}(y)-\Delta_j^{N} \le V_j^{\mathrm{obs}} < V_{N}(y) \right\}.
\end{align}
The target exposure and outcome are coupled identically, as in Lemma \ref{lem:ind_p_value_comparison}. Since the unaffected scores coincide under the coupling, the interventional and observational $p$-values satisfy, simultaneously for every $y\in\mathcal Y$,
\begin{align}
	p_{N,y}^{\mathrm{int}} \le p_{N,y}^{\mathrm{obs}}+\rho_{N,y}^{+} \qquad\text{a.s.}
\end{align}
Because $p_{N,y}^{\mathrm{obs}}>0$, the same threshold argument used in Appendix \ref{proof:cov_loss_ind_ite} gives
\begin{align}
	\left\{ p_{N,y}^{\mathrm{obs}} \le (\alpha-\rho_{N,y}^{+})_+ \right\} \subseteq \left\{ p_{N,y}^{\mathrm{int}} \le \alpha \right\}.
\end{align}
Evaluating this inclusion at $y=Y_{N}(1)$ yields
\begin{align}
	\Pr\left[ Y_{N}(1)\notin\Gamma_N^{\mathrm{IA-WCP+}} \right] \le \alpha.
\end{align}

\section{Additional Experiments}
\label{app:add_experiments}

\subsection{Traffic Slicing}
In this section, we evaluate the proposed methods in a wireless traffic-slicing problem \citep{foukas2017network,bao2017prediction}. We consider a wireless-network setting in which a finite amount of radio resources is allocated among multiple users. The users are organized into traffic slices representing services with similar requirements, such as latency-sensitive, high-throughput, or best-effort traffic. Each slice is assigned a resource pool that must be shared among its scheduled users. Resource allocation therefore induces interference, since scheduling one user reduces the resources available to other scheduled users in the same slice. Consequently, a scheduled user's outcome depends not only on its own channel and traffic characteristics, but also on the number of other scheduled users sharing the same resource pool \citep{foukas2017network,bao2017prediction}.

\subsubsection{Setup}
We consider a population of $N=200$ devices sharing radio resources across $N_S$ slices. Each device $i$ is characterized by a feature vector $X_i$ that includes its channel-quality indicator, traffic characteristics, backlog information, and the slice $Z_i$ to which it is assigned. The binary treatment $T_i\in\{0,1\}$ indicates whether device $i$ is scheduled and granted radio resources ($T_i=1$) or not ($T_i=0$). Scheduling decisions are drawn independently according to
\begin{align}
T_i\mid X_i\sim\operatorname{Bernoulli}(\pi(X_i)),
\end{align}
where the propensity score $\pi(X_i)$ is a function of the device covariates.

The neighborhood of device $i$ is the set of other devices assigned to the same slice,
\begin{align}
    \mathcal N_i(X_{1:N})=\{j\ne i:Z_j=Z_i\},
\end{align}
and the exposure of device $i$ corresponds to the number of users in this neighborhood that are scheduled,
 \begin{align}
E_i = \sum_{j\in \mathcal N_i(X_{1:N})} T_j.
\label{eq:traffic_slicing_exposure}
\end{align}
Thus, the exposure $E_i$ measures the same-slice contention experienced by device $i$. 

The potential outcome $Y_i(t)$ represents the throughput of device $i$ under the scheduling decision $t$. If device $i$ is not scheduled, $T_i=0$, it receives no radio resources and its throughput is zero, $Y_i(0)=0$. When device $i$ is scheduled, its throughput depends on its channel quality and the amount of same-slice contention. Specifically, this is given by
\begin{align}
Y_i(1) = \min\left\{ \frac{\eta(X_i)\xi_i}{(E_i+1)N_S}, R_{\max} \right\},
\label{eq:traffic_slicing_y1}
\end{align} 
where $\eta(X_i)$ is a spectral-efficiency term determined by the channel-quality indicator contained in $X_i$, and $\xi_i$ is a log-normal multiplicative noise variable normalized to have mean one. The factor $N_S^{-1}$ represents the fraction of the total resources assigned to each slice, while the term $(E_i+1)^{-1}$ models equal resource sharing among all scheduled devices in the slice, including device $i$. Consequently, greater exposure corresponds to stronger resource contention and lower achievable throughput. The throughput is capped at the maximum supported rate $R_{\max}$.

We use gradient-boosted quantile regression as the baseline predictor for the treated potential outcome. The predictor is trained on treated samples from five independently generated network realizations.  The resulting lower and upper quantile estimates, $\hat q_{\alpha/2}(X)$ and $\hat q_{1-\alpha/2}(X)$, define the exposure-independent nonconformity score of each treated calibration device as
\begin{align}
S(X_i,E_i,Y_i(1)) = \max\left\{ \hat q_{\alpha/2}(X_i)-Y_i(1), Y_i(1)-\hat q_{1-\alpha/2}(X_i) \right\}.
\end{align}

\subsubsection{Benchmarks}
We consider transductive counterfactual estimation and implement the following methods: $(a)$ uncalibrated quantile-regression interval (QR) \citep{koenker1978regression}, given by
\begin{align}
	\Gamma^{\rm QR}(X)=[\hat q_{\alpha/2}(X),\hat q_{1-\alpha/2}(X)],
\end{align}
 $(b)$ conventional WCP \citep{lei2021conformal}, $(c)$ ideal WCP, $(d)$ IA-WCP, and $(e)$ the stability-aware IA-WCP+. The last two methods are proposed in this work, while ideal WCP is the prediction set constructed from the interventional $p$-value in \eqref{eq:conf_p_value_trans_int}. Ideal WCP is available in simulation because the data-generating mechanism allows us to recompute the interventional exposures and outcomes, and therefore serves as an oracle baseline. For IA-WCP+, intervening to schedule the target device $I$ increases the exposure by one for every treated calibration device in the same slice, while leaving the other devices unaffected. Since the quantile-regression nonconformity score is $1$-Lipschitz in the outcome, the score-stability condition in Assumption \ref{ass:delta_score_stability_trans} is satisfied with
 
\begin{align}
	\label{eq:stab_exp_slicing}
	\Delta_j^I = \frac{Y_j(1)}{E_j+2} \mathds{1}\{Z_j=Z_I\}.
\end{align}

\subsubsection{Results}

\begin{figure}[t]
    \centering
    \includegraphics[width=\textwidth]{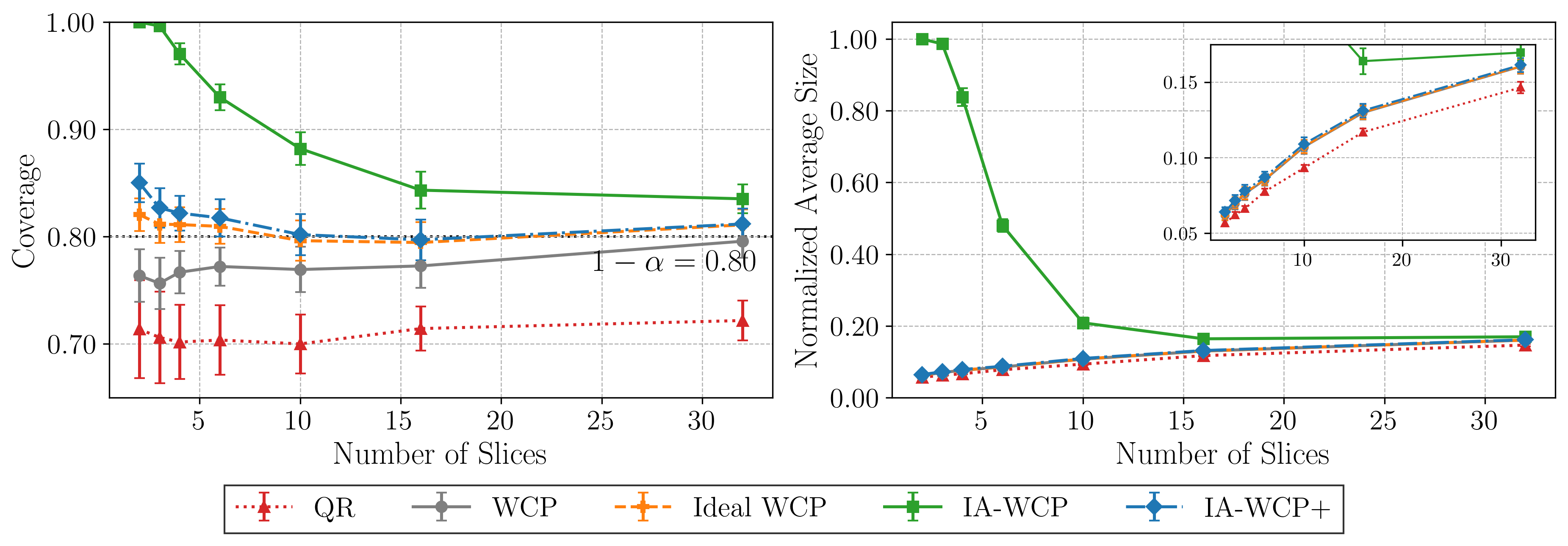}
	\caption{Empirical coverage (left) and average prediction-set size (right) for QR \citep{koenker1978regression}, WCP \citep{lei2021conformal}, ideal WCP based on interventional $p$-values, and the two proposed interference-adjusted methods, IA-WCP and IA-WCP+. The prediction sets target the treated potential outcomes of unscheduled devices. Prediction-set sizes are normalized by $R_{\max}$, and the target coverage is $1-\alpha=0.8$. Results are averaged over $250$ independently generated network realizations. The inset provides a magnified view of the prediction-set sizes for all methods except the IA-WCP procedure.}
    \label{fig:slicing}
\end{figure}

Figure \ref{fig:slicing} reports the empirical coverage and normalized average prediction-set size at the target miscoverage level $\alpha=0.2$, as the number of slices varies over $N_S\in\{2,3,4,6,10,16,32\}$. For each network realization, both quantities are averaged over all unscheduled target devices. A smaller number of slices $N_S$ corresponds to a larger average number of devices in each slice and stronger interference. By varying the number of slices $N_S$ we evaluate the performance of the proposed methods under different interference regimes. 

For every value of the number of slices $N_S$, the uncalibrated QR interval fails to attain the target coverage level $1-\alpha=0.8$.  WCP approaches nominal coverage when interference is weak, attaining coverage of approximately $0.81$ at $N_S=32$. As the number of slices decreases and the level of interference increases, however, its coverage falls below the nominal level, reaching approximately $0.76$ at $N_S=2$. In contrast, ideal WCP remains above the target coverage for all values of $N_S$. The coverage gap between WCP and ideal WCP illustrates the loss of validity caused by calibrating with observational scores.

Both proposed interference-adjusted methods maintain coverage above the target level for all values of $N_S$. IA-WCP uses a worst-case correction and therefore becomes highly conservative when a large fraction of the calibration weight is assigned to devices in the target slice. This conservativeness results in substantially larger prediction sets, particularly in the strongest-interference regimes. In contrast, IA-WCP+ mitigates this loss of efficiency by exploiting the device-specific stability margins in \eqref{eq:stab_exp_slicing}, and its coverage and prediction-set size remain close to those of ideal WCP.

\subsection{Additional Synthetic Results}
In this section, we consider the synthetic model introduced in Section \ref{sec:experiments} and present results that complement those in the main text. Specifically, we consider both transductive and inductive inference, additional target coverage levels, and experiments that vary the interference strength and outcome noise.

\subsubsection{Interference Strength}

\begin{figure}[t]
    \centering
    \begin{subfigure}[t]{0.49\textwidth}
        \centering
        \includegraphics[width=\textwidth]{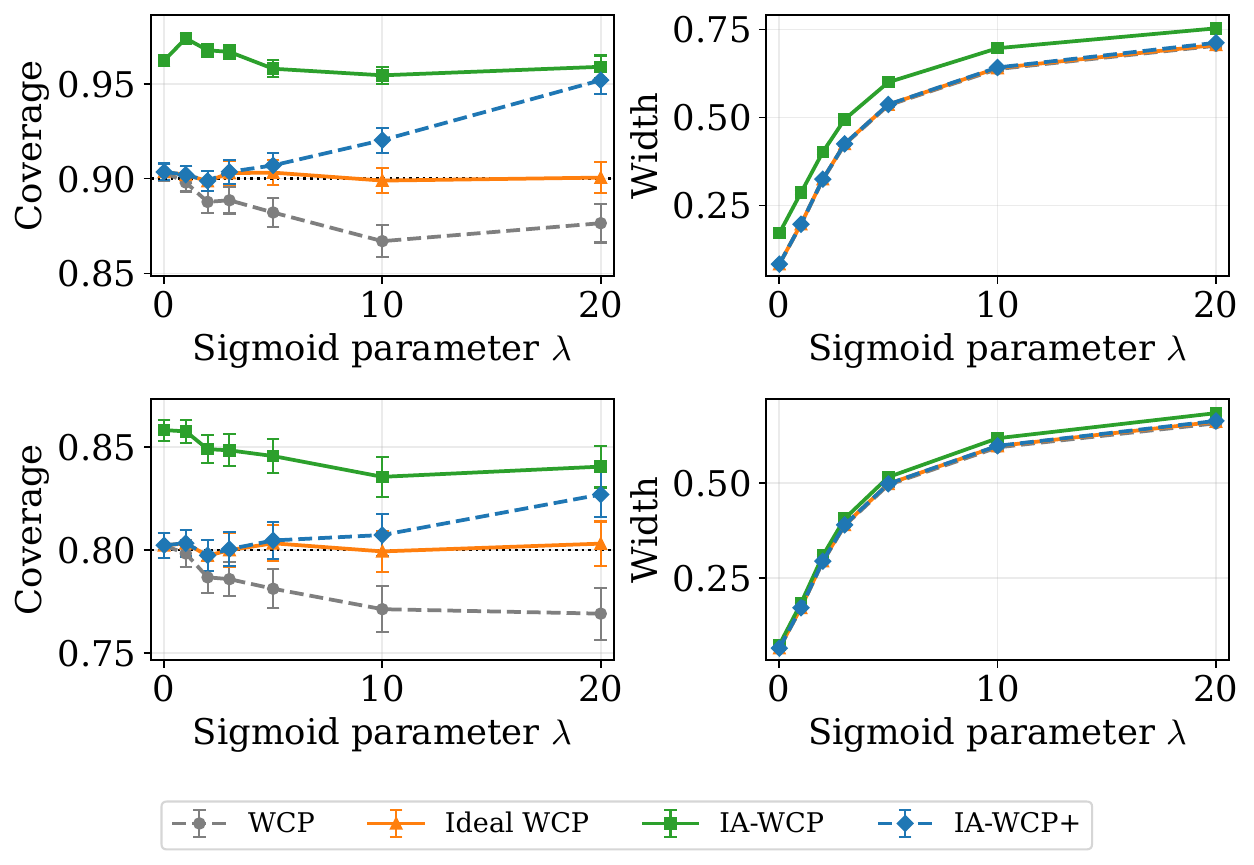}
        \caption{Transductive setting.}
        \label{fig:sweep_transductive}
    \end{subfigure}
    \hfill
    \begin{subfigure}[t]{0.49\textwidth}
        \centering
        \includegraphics[width=\textwidth]{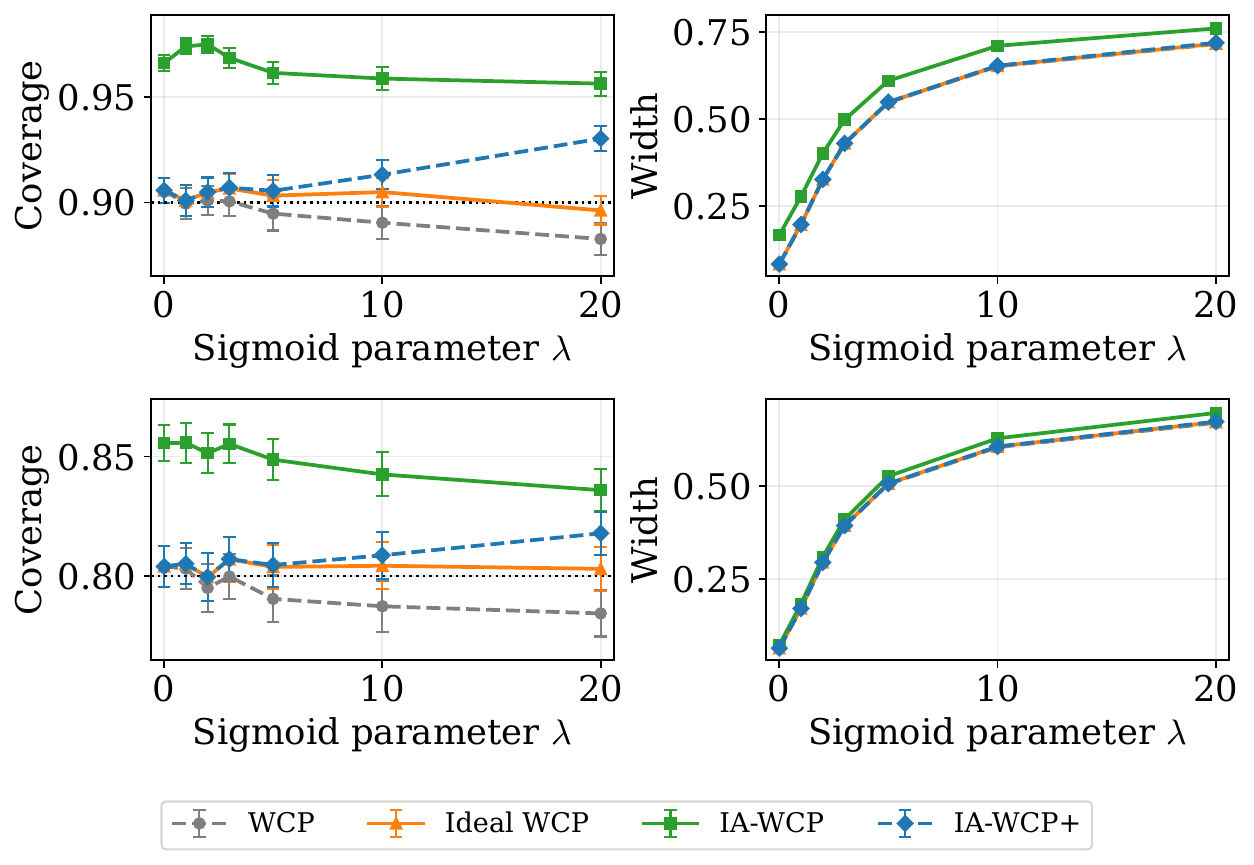}
        \caption{Inductive setting.}
        \label{fig:sweep_inductive}
    \end{subfigure}
    \caption{Coverage and normalized width of  WCP, ideal WCP, IA-WCP, and IA-WCP+ for transductive and inductive inference as a function of the sigmoid parameter $\lambda$. The first row corresponds to target miscoverage $\alpha=0.1$, while the second row corresponds to $\alpha=0.2$.}
    \label{fig:sweep_trans_ind}
\end{figure}

Figure \ref{fig:sweep_trans_ind} reports the empirical coverage and normalized width of prediction sets returned by WCP, ideal WCP, the proposed IA-WCP, and IA-WCP+ in both the transductive and inductive settings as a function of the sigmoid parameter $\lambda$.  The parameter $\lambda$ controls the sharpness of the exposure-outcome mechanism in \eqref{eq:outcome_ablation}.  In particular, when $\lambda$ is close to zero, the exposure-outcome mechanism is flat and peer exposure has little effect on the treated outcome. As $\lambda$ increases, the treated outcome changes rapidly depending on whether a majority of peers are treated.

For $\lambda=0$,  WCP and ideal WCP coincide and attain the target coverage levels. In this regime, the observational calibration scores used by  WCP are the same as the interventional calibration scores used by ideal WCP. However, as $\lambda$ increases, the gap between  WCP and ideal WCP becomes more visible, and  WCP begins to fall below nominal coverage. This degradation is more pronounced in the transductive setting. IA-WCP and IA-WCP+ remain above the target coverage level across the range of values considered.  WCP and ideal WCP remain relatively close in width, while IA-WCP becomes wider as the sigmoid parameter $\lambda$ grows because it applies a worst-case correction for affected calibration units. In contrast, IA-WCP+ maintains a width close to that of ideal WCP across the range of $\lambda$ values considered, while still attaining the target coverage level.

\subsubsection{Outcome Noise}

\begin{figure}[t]
    \centering
    \begin{subfigure}[t]{0.49\textwidth}
        \centering
        \includegraphics[width=\textwidth]{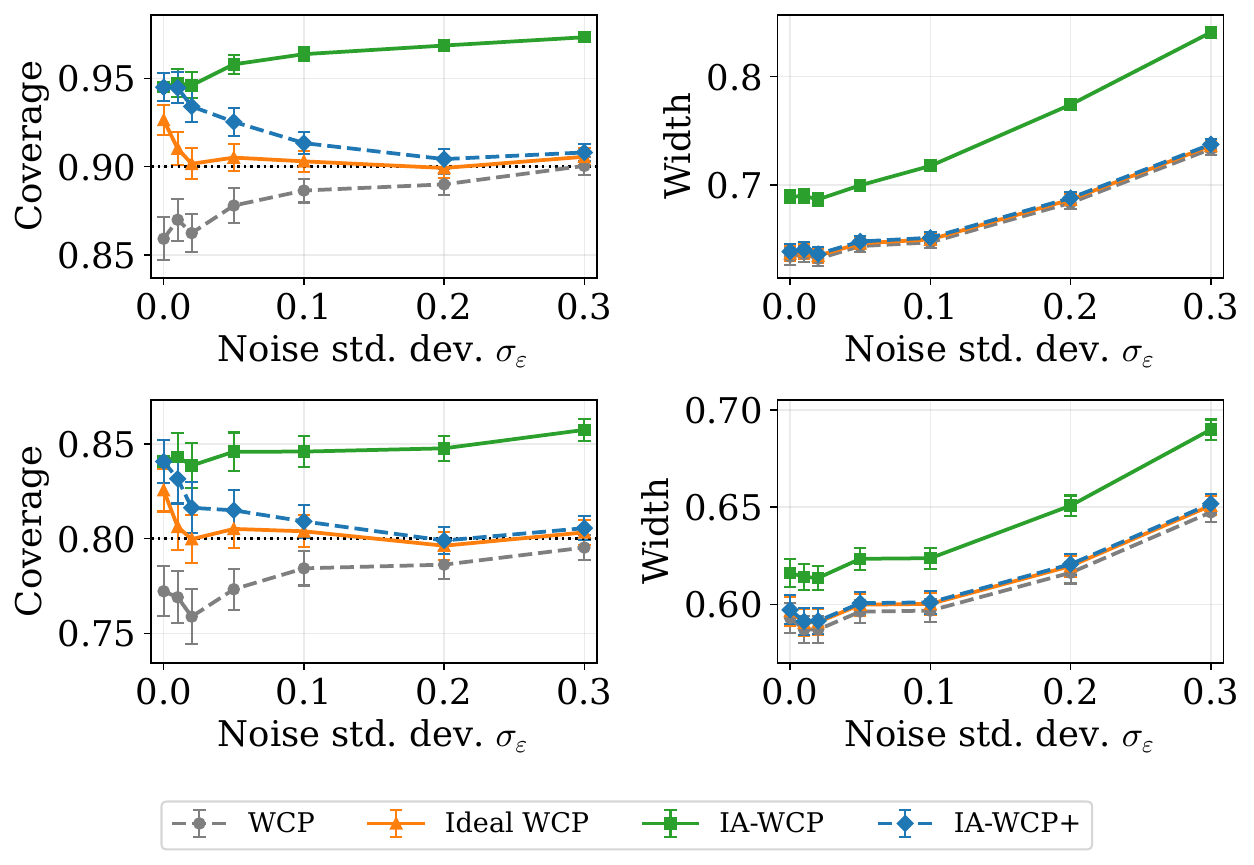}
        \caption{Transductive setting.}
        \label{fig:noise_transductive}
    \end{subfigure}
    \hfill
    \begin{subfigure}[t]{0.49\textwidth}
        \centering
        \includegraphics[width=\textwidth]{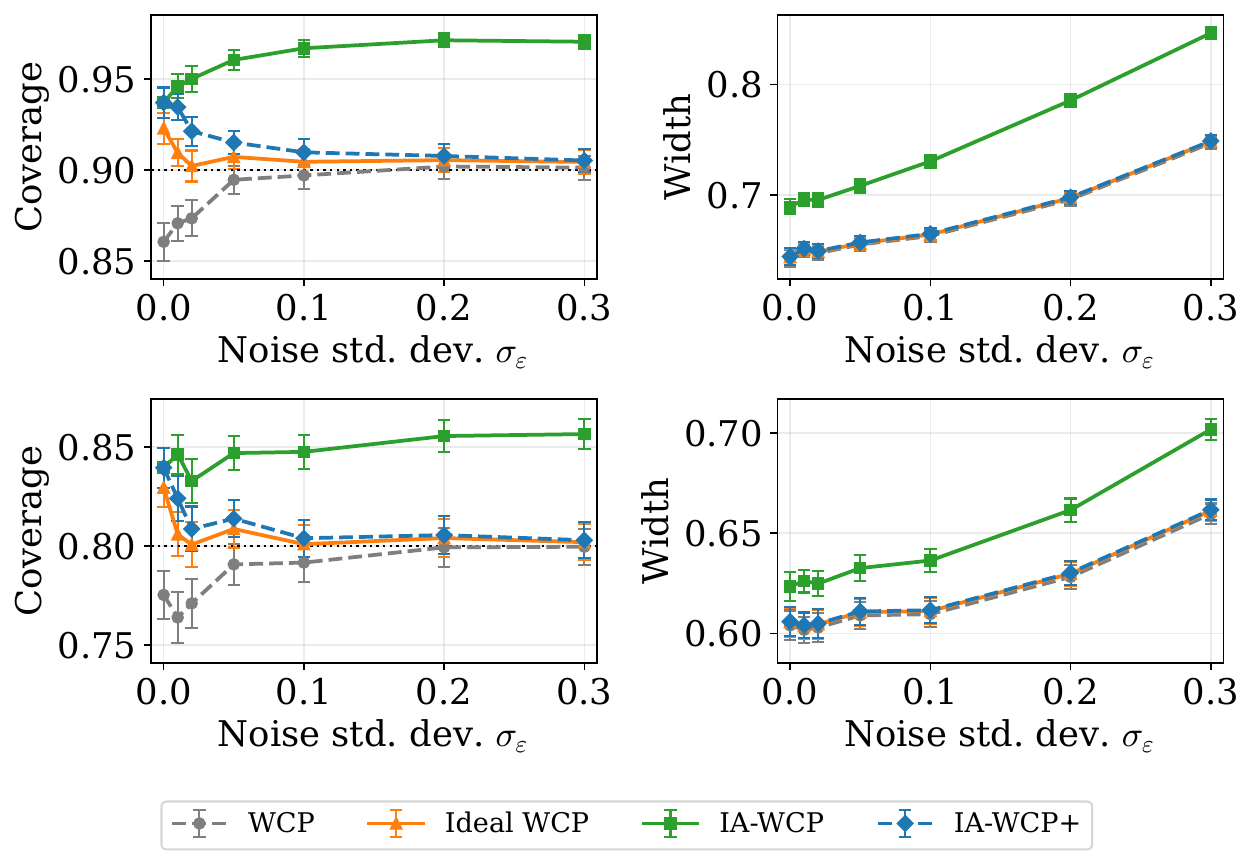}
        \caption{Inductive setting.}
        \label{fig:noise_inductive}
    \end{subfigure}
    \caption{Coverage and normalized width of  WCP, ideal WCP, IA-WCP, and IA-WCP+ for transductive and inductive inference as a function of the noise parameter $\sigma_\varepsilon$. The first row corresponds to target miscoverage $\alpha=0.1$, while the second row corresponds to $\alpha=0.2$.}
    \label{fig:sweep_noise}
\end{figure}

Figure \ref{fig:sweep_noise} studies the effect of the outcome noise level $\sigma_\varepsilon$. When the noise level is small, the nonconformity scores are determined primarily by the deterministic exposure mechanism, and changes in peer exposure induced by the target intervention can substantially alter the calibration scores. In this regime, WCP falls below the nominal coverage level, especially in the transductive setting.

As $\sigma_\varepsilon$ increases, the contribution of noise to the variability of the nonconformity scores grows. In this regime, WCP becomes closer to ideal WCP, and its empirical coverage degradation is reduced. IA-WCP and IA-WCP+ maintain coverage above the target level in both the transductive and inductive settings across the range of noise levels considered. Larger noise levels naturally lead to wider prediction sets for all methods, and the gap between IA-WCP and other baselines is most visible in the large-noise regime.
\subsubsection{Interference Locality}

\begin{figure}[t]
    \centering
    \begin{subfigure}[t]{0.49\textwidth}
        \centering
        \includegraphics[width=\textwidth]{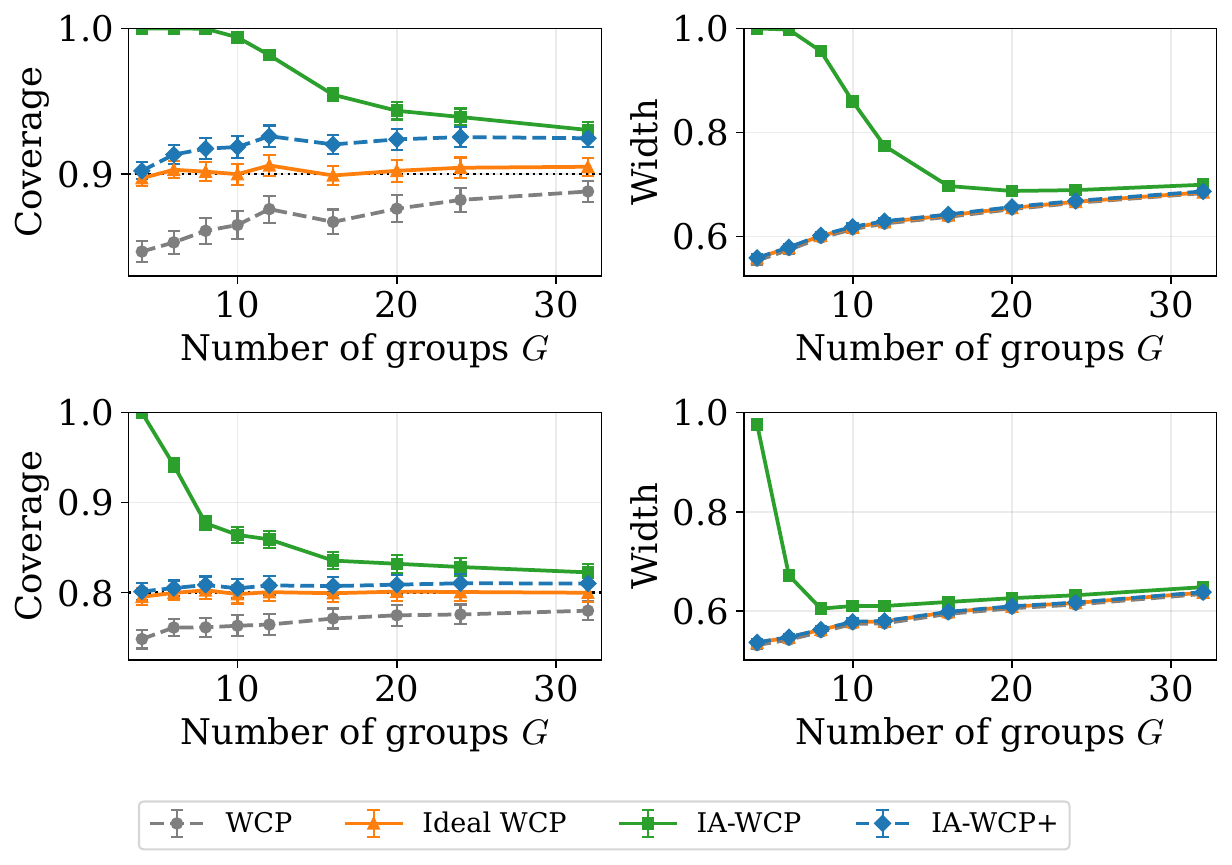}
        \caption{Transductive setting.}
        \label{fig:groups_transductive}
    \end{subfigure}
    \hfill
    \begin{subfigure}[t]{0.49\textwidth}
        \centering
        \includegraphics[width=\textwidth]{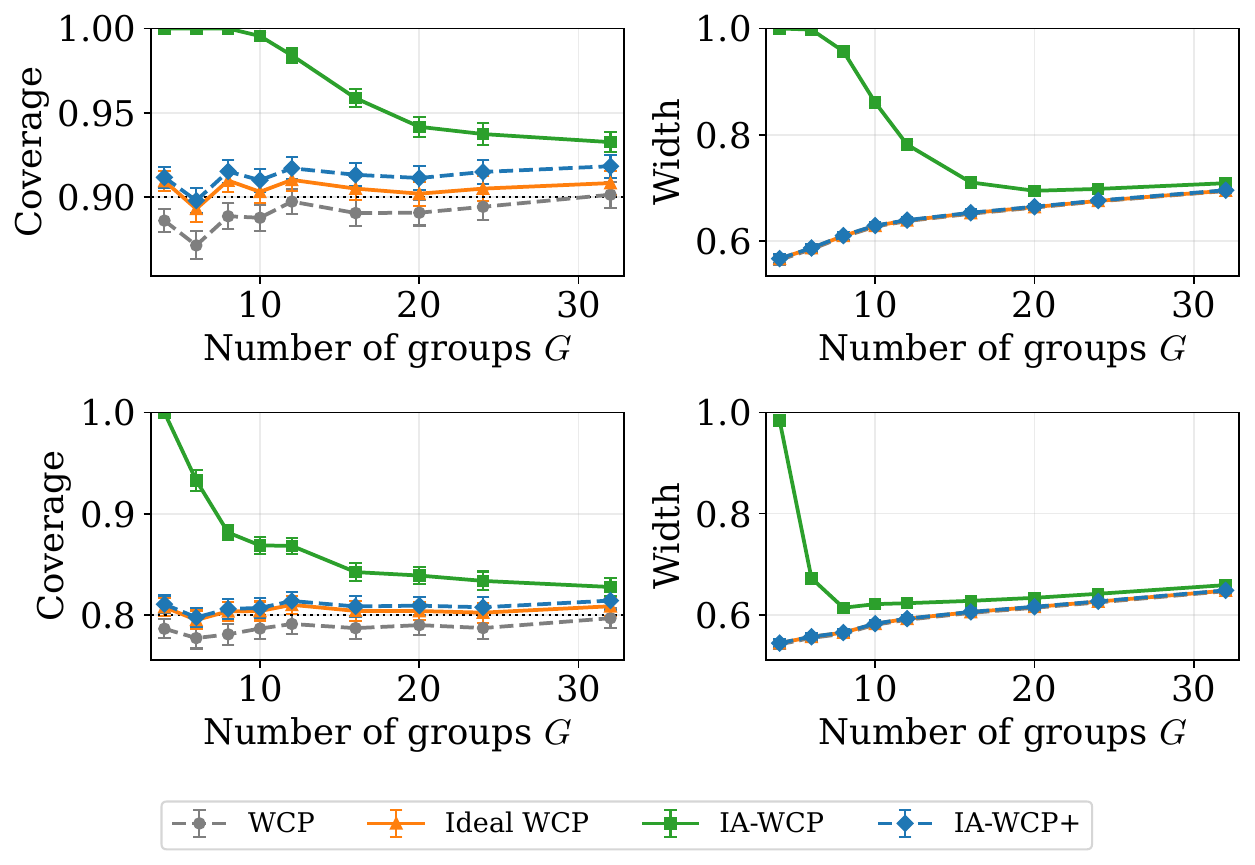}
        \caption{Inductive setting.}
        \label{fig:groups_inductive}
    \end{subfigure}
    \caption{Coverage and normalized width of WCP, ideal WCP, IA-WCP and IA-WCP+ for transductive and inductive inference as a function of the number of groups $G$. The first row corresponds to target miscoverage $\alpha=0.1$, while the second row corresponds to $\alpha=0.2$.}
    \label{fig:sweep_groups}
\end{figure}

Figure \ref{fig:sweep_groups} varies the number of groups $G$, which controls the locality of interference. For fixed $N$, smaller values of $G$ correspond to larger peer groups. In this regime, changing the treatment of a target unit, or embedding a new treated unit in the inductive setting, can affect a larger fraction of the calibration sample. The discrepancy between the observational and interventional calibration scores is therefore larger for smaller values of $G$.

For small values of $G$, WCP coverage falls below the nominal level. As $G$ increases, peer groups become smaller and interference becomes more localized. As a result, the observational calibration scores used by WCP become closer to the ideal interventional scores, and the empirical coverage of WCP improves, approaching that of ideal WCP. IA-WCP and IA-WCP+ compensate for the presence of interference, and their coverage levels remain above the nominal level throughout. As $G$ grows, the correction term in IA-WCP becomes smaller, so the conservativeness of the method decreases and its performance approaches that of ideal WCP.

\subsubsection{Base Predictors}
\label{sec:predictor_covariates}
\begin{figure}[t]
    \centering
    \begin{subfigure}[t]{0.49\textwidth}
        \centering
        \includegraphics[width=\textwidth]{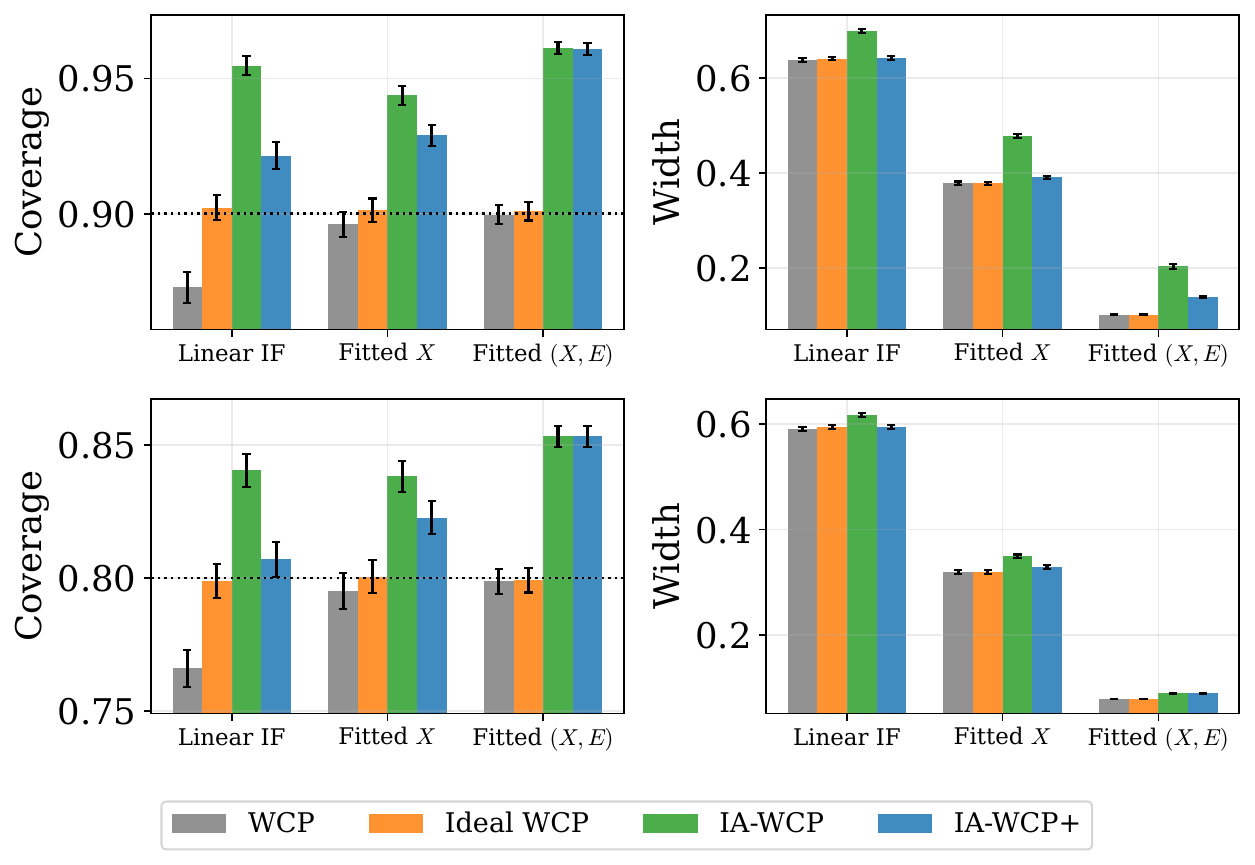}
        \caption{Transductive setting.}
        \label{fig:preds_transductive}
    \end{subfigure}
    \hfill
    \begin{subfigure}[t]{0.49\textwidth}
        \centering
        \includegraphics[width=\textwidth]{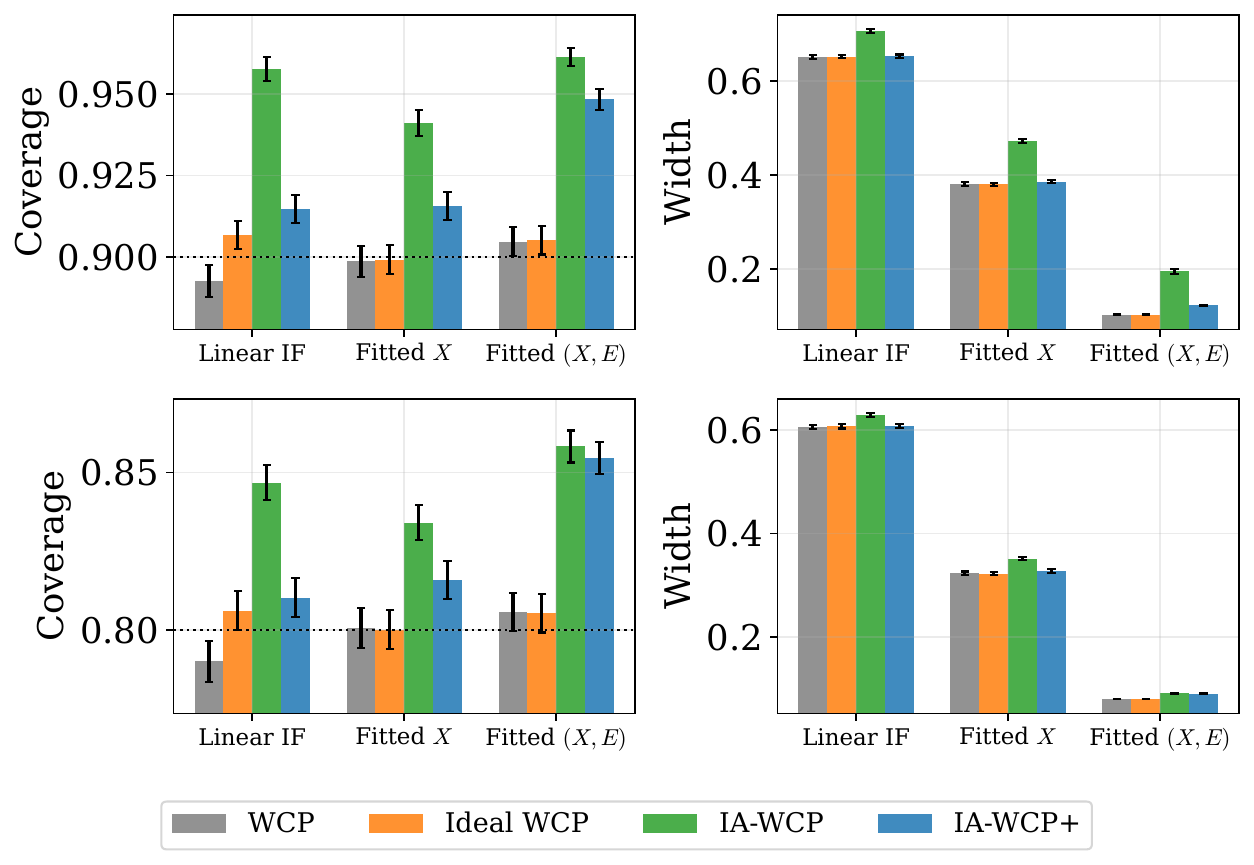}
        \caption{Inductive setting.}
        \label{fig:preds_inductive}
    \end{subfigure}
    \caption{Coverage and normalized width of WCP, ideal WCP, IA-WCP, and IA-WCP+ for transductive and inductive inference for different types of base predictors. The first row corresponds to target miscoverage $\alpha=0.1$, while the second row corresponds to $\alpha=0.2$.}
    \label{fig:sweep_preds}
\end{figure}

We conclude by evaluating the performance of WCP, ideal WCP, IA-WCP, and IA-WCP+ with the different base predictors introduced in Section \ref{subsec:exp_benchmarks}.

Figure \ref{fig:sweep_preds} reports the empirical coverage and normalized width of the prediction sets obtained using these methods and predictors. As expected, the mismatched interference-free predictor yields the widest prediction sets, whereas the predictor fitted using both $X$ and $E$ is the most efficient. We also observe that the miscoverage gap of WCP is larger when using the interference-free predictor. In contrast, with the fitted predictors, the coverage achieved by WCP is much closer to the target level, despite the absence of coverage guarantees. The proposed methods, IA-WCP and IA-WCP+, remain valid for all predictors considered, and their efficiency increases with the quality of the base predictor.
\end{document}